\documentclass[11pt]{article}

\usepackage{amssymb,amsmath,amsthm}
\usepackage{bbm}
\usepackage[noend]{algpseudocode}
\usepackage{graphicx}
\usepackage{verbatim}
\usepackage{url,xspace,hyperref}
\usepackage{cite}
\usepackage{caption}
\usepackage{subcaption}
\usepackage{multirow}
\usepackage{multicol}
\usepackage{latexsym}
\usepackage{amsmath,amssymb,enumerate}
\usepackage{algorithm,algorithmicx}
\usepackage{float}
\usepackage{xcolor}
\usepackage{mathrsfs}
\usepackage{cleveref}
\usepackage{booktabs}
\usepackage{breqn}

\usepackage{fullpage}
\usepackage{geometry}
\usepackage{setspace}
\usepackage{tcolorbox}
\usepackage{xspace}

\tcbuselibrary{skins,breakable}
\tcbset{enhanced jigsaw}

\newtheorem{theorem}{Theorem}[section]
\newtheorem*{theorem*}{Theorem}
\newtheorem{definition}{Definition}[section]

\newtheorem{lemma}[theorem]{Lemma}

\newtheorem{corollary}[theorem]{Corollary}

\newtheorem{fact}{Fact}[section]

\newcounter{note}[section]

\newcommand{\pr}{\mathbf{P}} 
\newcommand{\E}{\mathbb{E}}

\newcommand{\A}{\mathcal{A}}

\newcommand{\ignore}[1]{}

\newcommand{\half}{\frac{1}{2}}
\newcommand{\B}{\mathcal{B}}

\newcommand{\rucb}{$\mathtt{RUCB}$}

\newcommand{\balg}{$\mathtt{C2B}$}
\newcommand{\rdel}{r(\delta)}

\newenvironment{tbox}{\begin{tcolorbox}[
		enlarge top by=5pt,
		enlarge bottom by=5pt,
		 breakable,
		 boxsep=0pt,
                  left=4pt,
                  right=4pt,
                  top=10pt,
                  arc=0pt,
                  boxrule=1pt,toprule=1pt,
                  colback=white
                  ]
	}
{\end{tcolorbox}}

\newif\ifConfVersion

\title{An Asymptotically Optimal Batched Algorithm \\ for the Dueling Bandit Problem} 
\author{Arpit Agarwal \and Rohan Ghuge \and Viswanath Nagarajan}
\begin{document}

\maketitle

\begin{abstract}
We study the $K$-armed dueling bandit problem, a 
variation of the traditional multi-armed bandit problem in which feedback is obtained in the form of pairwise comparisons.
Previous learning algorithms have focused on the \emph{fully adaptive} setting,  where the algorithm can make updates
after every comparison.
The ``batched'' dueling bandit problem is motivated by large-scale applications like web search ranking and recommendation systems, where  performing sequential updates may be infeasible.
In this work, we ask: {\em is there a solution using only a few adaptive rounds
that matches the asymptotic regret bounds of the best sequential algorithms for  $K$-armed dueling bandits?}
We answer this in the affirmative \emph{under the Condorcet condition}, a standard setting of the $K$-armed dueling bandit problem.
We obtain asymptotic regret of $O(K^2\log^2(K)) + O(K\log(T))$ in $O(\log(T))$ rounds, where $T$ is the time horizon. 
Our regret bounds nearly match 
the best regret bounds known in the fully sequential setting under the Condorcet condition.
Finally, in
computational experiments  over a variety of real-world datasets,
we observe that our algorithm
using $O(\log(T))$ rounds
achieves almost the same performance as fully sequential algorithms (that use $T$ rounds).
\end{abstract}

\section{Introduction}
\ifConfVersion
\vspace{-0.1in}
\else
\fi
The \emph{$K$-armed dueling bandit} problem is a variation of the traditional multi-armed bandit problem in which feedback is obtained in the form of pairwise preferences.
This problem has applications in a wide-variety of domains like search ranking, recommendation systems and sports ranking where eliciting qualitative feedback is easy while real-valued feedback is not easily interpretable; 
thus, it has been a popular topic of research in the  machine learning community
(see, for example, \cite{YueJo11, YueBK+12, Urvoy+13, Ailon+14, Zoghi+14, Zoghi+15, Zoghi+15a, Dudik+15, Jamieson+15, Komiyama+15a, Komiyama+16, Ramamohan+16, ChenFr17}).

Previous learning algorithms have focused on a \emph{fully adaptive} setting; 
that is, the learning algorithm can make updates 
in a sequential fashion.
Such updates might be impractical in large systems; for example, consider web-search ranking where the goal is to provide a list (usually \emph{ranked}) of candidate documents to the user of the system in response to a query~\cite{RadlinskiKJ02, Joachims02, YueJoachims09, HofmannWR13}.
Modern day search engines use hundred of parameters to compute a ranked list in response to a query, and online learning frameworks (based on user feedback) have been invaluable in automatically tuning these parameters~\cite{Liu09}. 
However, given the scale of the system, it may be infeasible to adapt after each interaction: users may make multiple queries in a short time or multiple users may simultaneously query the system. Hence, we prefer solutions with \emph{limited rounds} of adaptivity. Concretely, we ask: {\em is there a solution using only a few adaptive rounds
that matches the asymptotic regret bounds of the best sequential algorithms for  $K$-armed dueling bandits?} 

This ``batched'' dueling bandit problem was introduced recently  in \cite{AgarwalGN22}. Here, the learning algorithm's actions are partitioned into a limited number of rounds. In each round/batch, the algorithm
commits to a \emph{fixed set} of  pairwise comparisons, and the feedback 
for all these comparisons  is received simultaneously. Then, the   algorithm uses the feedback from the current batch of comparisons to choose  comparisons for the next batch. \cite{AgarwalGN22}  studied this problem under two different conditions: (i) the strong stochastic transitivity  and stochastic 
triangle inequality (SST+STI) condition, which enforces a certain linear ordering over the arms;  
(ii) the Condorcet condition, which requires one arm to be superior to all others. 
Under SST+STI, their work provided almost  tight upper and lower bounds on the trade-off between number of rounds and regret; in particular, 
they showed that one can achieve  
worst-case  regret of $O(K \log^2 T)$ using $\Theta(\log T)$
rounds ($T$ is the time-horizon).\footnote{They also gave a more complicated algorithm with regret $O(K\log^2K \, \log T)$ in $O(\log T + \log K\, \log\log K)$ rounds, under the SST+STI condition.}
Under the Condorcet condition, which is more general 
than SST+STI, they  achieved a regret upper bound of
$O(K^2 \log T)$  in $O(\log T)$ rounds. 
Previous work \cite{Zoghi+14,Komiyama+15a} on fully sequential algorithms has shown that  it is possible to achieve an
\emph{asymptotic} upper bound of $O(K^2 + K \log T)$ under the Condorcet condition. Very
recently, \cite{SahaGi22} improved the sequential regret bound even further by obtaining regret   $O(K\log T)$, which is the best possible even in the special case of SST+STI~\cite{YueBK+12}.
In the batched setting, the upper bound of \cite{AgarwalGN22}
does not achieve this asymptotic 
optimality, irrespective of the number of batches, 
due to the presence of a $K^2$ {\em multiplicative} factor in the regret bound. 
Their work left open the possibility  of obtaining a batched algorithm   achieving   asymptotic optimality under  the Condorcet condition. In this paper, we nearly resolve this question, by providing an algorithm with $O(K^2\log^2K + K \log T)$ regret in $\Theta(\log T)$ rounds, under the Condorcet condition.

\subsection{Contributions}
\begin{itemize}
    \item  We 
    design an algorithm, denoted \balg, for the batched dueling bandit problem, and analyze its regret under the Condorcet condition. This 
    algorithm achieves a smooth trade-off between the expected regret and the number of batches, $B$. 
    
    \item Crucially, when $B = \log(T)$, our regret bounds nearly match the best regret bounds   \cite{Komiyama+15a, Zoghi+14} known in the fully sequential setting. Hence, our results show that 
    $O(\log T)$ rounds are sufficient to achieve asymptotically optimal regret as a function of $T$.

    \item Our results rely on new ideas for showing that the 
    Condorcet winner arm can be `trapped' using few adaptive rounds with high (constant) probability while incurring a reasonable amount of regret. 
    We can then integrate over this space of probabilities to obtain a bound on the expected regret (in the same vein as \cite{Zoghi+14}). 
    Once the Condorcet arm is `trapped', we can quickly eliminate all other `sub-optimal' arms
    and minimize regret in the process.
    
    \item Finally, we run computational experiments to validate our theoretical results.
    We show that \balg, using $O(\log T)$ batches, achieves almost the same performance as fully sequential algorithms (which effectively use $T$ batches) over a variety of real datasets.
\end{itemize}

\subsection{Preliminaries}\label{sec:prelim}
\ifConfVersion
\vspace{-0.1in}
\else
\fi
The \emph{$K$-armed dueling bandit} problem \cite{YueBK+12} is an online optimization problem, where the goal is to find the best among $K$ bandits $\B = \{1, \ldots, K\}$ using noisy pairwise comparisons with low \emph{regret}.
In each time-step,  a noisy comparison between two arms (possibly the same), say $(i, j)$, is performed. 
The outcome of the comparison is an independent random variable, and the probability of picking $i$ over $j$ is denoted $p_{i, j} = \frac{1}{2} + \Delta_{i, j}$ where $\Delta_{i, j} \in (-\frac{1}{2}, \frac{1}{2})$. 
Here, $\Delta_{i, j}$ can be thought of as a measure of distinguishability between the two arms,
and we use $i \succ j$ when $\Delta_{i, j} > 0$. We also refer to $\Delta_{i, j}$ as the \emph{gap} between $i$ and $j$.

This problem has been studied under various 
conditions on the pairwise probabilities $p_{i,j}$'s.
One such condition is the {\em strong stochastic transitivity} and {\em stochastic triangle inequality} (SST+STI) condition where there exists an ordering over arms, denoted by $\succeq$, such that for every triple $i \succeq j \succeq k$, we have $ \Delta_{i, k} \geq \max\{\Delta_{i,j}, \Delta_{j, k}\}, $
and 
$\Delta_{i, k} \leq \Delta_{i,j} + \Delta_{j, k}$ \cite{YueBK+12, YueJo11}.
In this paper, we work under the well-studied Condorcet 
winner condition, which is much more general than the 
SST+STI condition \cite{Urvoy+13,Zoghi+14, Komiyama+15a}.
We say that arm $i$ is a \emph{Condorcet winner} if, and only if, $p_{i, j} > \frac{1}{2}$ for all $j \in \B \setminus \{i\}$.
The \emph{Condorcet condition} means that there exists a Condorcet winner.

Throughout the paper, we let $a^*$ refer to the Condorcet arm. To further simplify notation, we define $\Delta_j = \Delta_{a^*, j}$; that is, the gap between $a^*$ and $j$. We define the \emph{regret} per time-step as follows: suppose arms $i_t$ and $j_t$ are chosen in time-step  $t$, then the regret $r(t) = \frac{\Delta_{i_t} + \Delta_{j_t}}{2}$. 
The cumulative regret up to time $T$ is $R(T) = \sum_{t=1}^T r(t)$, where $T$ is the time horizon, and it's assumed that $K \leq T$.  
The cumulative regret can be equivalently stated as $R(T) = \frac{1}{2}\sum_{j=1}^{K} T_j \Delta_{j}$, where $T_j$ denotes the number comparisons involving arm $j$.
The goal of an algorithm is to minimize the cumulative 
$R(T)$.
We define $\Delta_{\min} = \min_{j:\Delta_j >0}\Delta_j$ to be the smallest non-zero gap of any arm with $a^*$.

\subsection{Batch Policies}
\ifConfVersion
\vspace{-0.1in}
\else
\fi
In traditional bandit settings, actions are performed \emph{sequentially}, utilizing  the results of \emph{all prior actions} in determining the next action. In the batched setting, the algorithm must commit to a round (or \emph{batch}) of actions to be performed  \emph{in parallel}, and can only observe the results after all   actions in the batch have been performed. More formally, given a number $B$ of batches, the algorithm proceeds as follows. In each batch $r = 1, 2, \ldots B$, the algorithm first decides on the comparisons to be performed; then,  \emph{all}  outcomes of the batch-$r$  comparisons  are received simultaneously. The algorithm can then, \emph{adaptively}, select the next batch of comparisons. 
Note that even the size of the next 
batch can be  adaptively  decided based on the observations in previous batches.
Finally,
the total number of comparisons (across all batches) must sum to $T$. We assume that the values of $T$ and $B$ are known. Observe that when $T = B$, we recover the fully sequential setting.

\subsection{Results and Techniques}
\ifConfVersion
\vspace{-0.1in}
\else
\fi
Given any integer $B \geq 1$, we obtain a $B$-round algorithm for the dueling bandit problem. We provide both  high-probability and expected regret bounds, stated in the following theorems.

\begin{theorem}\label{thm:condorcet-init-high-prob}
For any integer $B \geq 1$, there is an algorithm for the $K$-armed dueling bandit problem
that uses at most $B$ rounds with the following guarantee. For any $\delta > 0$, with probability at least $1 - \delta - \frac{1}{T}$, its regret   under the Condorcet
condition is at most 
\ifConfVersion
\begin{align}
     R(T) \leq  O\left( T^{1/B} \cdot \frac{K^2 \log(K)}{\Delta_{\min}^2} \cdot \log\left(\frac{\log K}{\Delta_{\min}}\right) \right) &\,+\, O\left( T^{2/B} \cdot K^2\cdot \sqrt{\frac{1}{\delta}}\right) \notag \\
     &+\, \sum_{j \neq a^*} O\left(\frac{T^{1/B}\cdot \log(KT)}{\Delta_j}\right). \notag
\end{align}
\else
$$ R(T) \leq  O\left( T^{1/B} \cdot \frac{K^2 \log(K)}{\Delta_{\min}^2} \cdot \log\left(\frac{\log K}{\Delta_{\min}}\right) \right) \,+\, O\left( T^{2/B} \cdot K^2\cdot \sqrt{\frac{1}{\delta}}\right) \,+\, \sum_{j \neq a^*} O\left(\frac{T^{1/B}\cdot \log(KT)}{\Delta_j}\right).$$
\fi
\end{theorem}

\begin{theorem}\label{thm:condorcet-init-exp}
For any integer $B \geq 1$, there is an algorithm for the $K$-armed dueling bandit problem 
that uses at most $B$ rounds, with expected regret   under the Condorcet 
condition at most 
\ifConfVersion
\begin{align}
    \E[R(T)] =  O\left( T^{1/B} \cdot \frac{K^2 \log(K)}{\Delta_{\min}^2} \cdot \log\left(\frac{\log K}{\Delta_{\min}}\right) \right) &\,+\, O\left( T^{2/B} \cdot K^2\right) \notag \\
    &+\, \sum_{j \neq a^*} O\left(\frac{T^{1/B}\cdot \log(KT)}{\Delta_j}\right).\notag
\end{align}
\else
$$\E[R(T)] =  O\left( T^{1/B} \cdot \frac{K^2 \log(K)}{\Delta_{\min}^2} \cdot \log\left(\frac{\log K}{\Delta_{\min}}\right) \right) \,+\, O\left( T^{2/B} \cdot K^2\right) \,+\, \sum_{j \neq a^*} O\left(\frac{T^{1/B}\cdot \log(KT)}{\Delta_j}\right).$$
\fi
\end{theorem}

When the number of rounds $B=\log(T)$, we obtain a batched algorithm that achieves the asymptotic optimality (in terms of $T$), even for sequential algorithms. We formalize this observation in the following corollary.

\begin{corollary}\label{cor:log}
There is an algorithm for the $K$-armed dueling bandit problem 
that uses at most $\log(T)$ rounds, 
with expected regret   under the Condorcet 
condition  at most 
$$\E[R(T)] =  O\left( \frac{K^2 \log(K)}{\Delta_{\min}^2} \cdot \log\left(\frac{\log K}{\Delta_{\min}}\right) \right) \,+\, \sum_{j \neq a^*} O\left(\frac{ \log(KT)}{\Delta_j}\right).$$
\end{corollary}
By a lower-bound result from \cite{AgarwalGN22}, it follows that no algorithm can achieve  $O(\frac{K}{\Delta_{\min}}\cdot poly(\log T))$ regret using $o(\frac{\log T}{\log\log T})$ rounds, even under the SST+STI condition. So, the $O(\log T)$ rounds required to achieve asymptotic optimality in Corollary~\ref{cor:log} is nearly the best possible.

\ifConfVersion
{\bf Technical Challenges.}
\else
\paragraph{Technical Challenges.}\fi The only prior approach for batched dueling bandits (under the Condorcet condition) is the algorithm PCOMP from \cite{AgarwalGN22}, which  performs all-pairs comparisons among arms in an active set. Such an approach cannot achieve regret better than $O(K^2 \log T)$ because the active set may remain large throughout. In order to achieve better regret bounds, \cite{AgarwalGN22} focus on the stronger SST+STI condition. In this setting, their 
main idea  is to first sample a \emph{seed set}, and use this seed set to eliminate sub-optimal arms. Their algorithm proceeds by performing  all pairwise comparisons  between the seed set and the set of active arms. However, the analysis of these `seeded comparison' algorithms  crucially rely on the total-ordering  imposed by the SST and STI assumptions. Unfortunately, there is no such structure to exploit in the Condorcet setting: if the seed set does not contain the Condorcet winner, we immediately incur high regret. 

The existing fully sequential algorithms such as RUCB \cite{Zoghi+14} and 
RMED \cite{Komiyama+15a}
are \emph{highly adaptive} in nature.
For instance, RUCB plays each \emph{candidate} arm against an optimistic \emph{competitor}
arm using upper confidence bounds (UCB) on pairwise  
probabilities.
This allows RUCB to quickly filter out candidates
and uncover the Condorcet arm.
Similarly, RMED plays each arm against a carefully selected \emph{competitor} arm that is likely 
to beat this arm.
However, such \emph{competitors} can change frequently over trials 
in both RUCB and RMED.
Since the batched setting requires 
comparisons 
to be predetermined,
we do not have the flexibility to adapt to such changes in \emph{competitors}.
Hence, these existing fully sequential algorithms cannot 
be easily implemented in our setting.

Furthermore, we might also be tempted to consider an
\emph{explore-then-exploit} strategy where we first \emph{explore} to find the Condorcet arm
and \emph{exploit} by playing this arm for remaining trials.
However, this strategy is likely to fail  because identifying the Condorcet arm with high probability might involve performing many comparisons, directly leading to high ($\Omega(K^2 \log T)$) regret; on the other hand, if the Condorcet winner is not identified with high probability, the exploit phase becomes expensive. This motivated us to consider algorithms that allow some form of \emph{recourse}; that is, unless an arm is found to be sub-optimal, it must be given the opportunity to participate in the comparisons (as it could be the Condorcet winner). 

The idea behind our algorithm is to identify the Condorcet winner $a^*$ in a small {\em expected} number of rounds, after which it uses this arm as an ``anchor'' to eliminate sub-optimal arms while incurring low regret. 
To identify the best arm, in each round we define a candidate arm and compare it against arms that it ``defeats''. Arms that are not defeated by the candidate arm are compared to \emph{all} active arms: this step ensures that the Condorcet winner is eventually discovered. 
We show that $a^*$ becomes the candidate, and \emph{defeats all other arms} within a small number of rounds (though the algorithm may not know if this has occurred). 
Additionally, once this condition is established, 
it remains invariant in future rounds. This allows us to eliminate sub-optimal arms and achieve low regret.

\ifConfVersion
{\bf Comparison to \rucb.}
\else
\paragraph{Comparison to \rucb.}\fi Initially, \rucb \ puts all arms in a pool of potential champions,  and ``optimistically'' (using a upper confidence bound) performs all pairwise comparisons.
Using these, it constructs a set of candidates $C$.
If $|C| = 1$, then that arm is the hypothesised Condorcet winner and placed in a set $B$.
Then, a randomized strategy is employed to choose 
a champion arm $a_c$ (from sets $C$ and $B$) which is compared to arm $a_d$ which is most likely to beat it. 
The pair $(a_c, a_d)$ is compared, the probabilities are updated and the algorithm continues.
Although our algorithm also seeks to identify the best arm, we do not employ the UCB approach nor do we use any randomness. 
In their analysis, \cite{Zoghi+14} show that the best arm eventually enters the set $B$, and remains in $B$: we also show a similar property for our algorithm in the analysis.
Finally,
similar to the analysis of \cite{Zoghi+14}, we first give a high-probability regret bound for our algorithm which we then convert to a bound on the expected regret.

\ifConfVersion
\else
\paragraph{Organization.} The rest of the paper is organized as follows. In \S\ref{sec:related-work}, we discuss related work on the $K$-armed dueling bandit problem. We present our batched algorithm in \S\ref{sec:alg}, and present computational results in \S\ref{sec:computations}. We conclude in \S\ref{sec:conclusion}.
\fi

\ifConfVersion
\vspace{-0.1in}
\else
\fi
\section{Related Work}\label{sec:related-work}
\ifConfVersion
\vspace{-0.1in}
\else
\fi
The $K$-armed dueling bandit problem has been widely studied in recent years (we refer the 
reader to \cite{SuiZHY18} for a comprehensive survey). 
Here, we survey the works that are most closely related to our setting.
This problem was first studied in 
\cite{YueBK+12} under the SST and STI setting.
The authors obtained
a worst-case regret upper bound of $\widetilde{O}(K\log T/\Delta_{\min})$ and provided a matching lower bound.
\cite{YueJo11} considered a slightly more general version of the 
SST and STI setting and achieved an instance-wise optimal 
regret upper bound of $\sum_{j: \Delta_j > 0} O\left({\log (T)}/{\Delta_{j}}\right)$.
Since, the SST+STI condition imposes a total order
over the arms and might not hold for real-world datasets, 
\cite{Urvoy+13} initiated the study of dueling bandits under the Condorcet winner condition.
\cite{Urvoy+13} 
proved a $O(K^2\log T/\Delta_{\min})$ regret upper bound
under the Condorcet condition, which was 
improved by \cite{Zoghi+14} 
to  $O(K^2/\Delta_{\min}^2) + \sum_{j: \Delta_j > 0}O( \log T/\Delta^2_{j})$.
\cite{Komiyama+15a} achieved a similar but tighter KL divergence-based bound, which is shown to be \emph{asymptotically instance-wise optimal} (even in terms constant factors).
There are also other works that improve the dependence on $K$ in the upper bound, but 
suffer a worse dependence on $\Delta_j$'s \cite{Zoghi+15a}.
This problem has also been studied 
under other noise models such as utility based models \cite{Ailon+14}
and other notions of regret \cite{ChenFr17}.
Alternate notions of winners
such as Borda winner \cite{Jamieson+15}, Copeland winner \cite{Zoghi+15, Komiyama+16, WuLiu16}, and von Neumann winner  \cite{Dudik+15} have also been considered.
There are also several works on extensions of dueling bandits that allow multiple arms to be compared at once \cite{Sui+17, Agarwal+20, SahaG19}.

All of the aforementioned works on the dueling bandits problem are limited to the sequential setting.
Recently, \cite{AgarwalGN22} initiated the 
study of the 
batched version of the $K$-armed dueling bandits. 
Their main results are under the SST and STI setting. They give two algorithms, called SCOMP and SCOMP2, for the batched $K$-armed dueling bandit problem. 
For any integer $B$, 
SCOMP uses at most $B+1$ batches and has expected regret bounded by $\sum_{j: \Delta_{j} > 0}O({\sqrt{K}T^{1/B} \log(T)}/{\Delta_{j}})$. 
When $B=\log (T)$, this nearly matches (up to a factor of $\sqrt{K}$) the best known instance-dependent regret bound of $\sum_{j: \Delta_j > 0} O({\log(T)}/{\Delta_j})$ obtained by \cite{YueBK+12}.
SCOMP2 aims to achieve better worst-case regret: 
it
uses at most $2B+1$ batches,
and  
has regret  
$  O\left({KB T^{1/B} \log(T)}/{\Delta_{\min}}\right).$
Thus, when $B = \log(T)$, the expected worst-case regret is  $O\left({K \log^2(T)}/{\Delta_{\min}}\right)$, matching the best known result in the sequential setting up to an additional logarithmic factor. Under the Condorcet condition, \cite{AgarwalGN22} give a straightforward pairwise comparison algorithm (PCOMP), that achieves expected regret bounded by $O(K^2\log(T)/\Delta_{\min})$ in $\log(T)$ batches.
They also provide a nearly matching lower bound of $\Omega(\frac{KT^{1/B}}{B^2\Delta_{\min}})$ for any $B$-batched algorithm. This implies that our bound (for $B$-round algorithms) in \Cref{thm:condorcet-init-exp} cannot be significantly improved.

\ifConfVersion
\else
Batched processing for the stochastic multi-armed bandit (MAB) problem  has been investigated extensively in the past few years.
A special case when there are two bandits was studied by~\cite{PerchetRC+16}. They obtain a worst-case regret bound of $O\left(\left(\frac{T}{\log(T)}\right)^{1/B}\frac{\log(T)}{\Delta_{\min}}\right)$.
\cite{Gao+19} studied the general problem and obtained a worst-case regret bound of 
$O\left(\frac{K\log(K)T^{1/B}\log(T)}{\Delta_{\min}}\right)$, which was later improved by \cite{EsfandiariKM+21} to $O\left(\frac{KT^{1/B} \log(T)}{\Delta_{\min}}\right)$. Furthermore, \cite{EsfandiariKM+21} obtained an instance-dependent regret bound of $\sum_{j:\Delta_j > 0}T^{1/B}O\left(\frac{\log(T)}{\Delta_j}\right)$.
Our results for batched dueling bandits are of a similar flavor; that is, we get a similar dependence on
$T$ and $B$.
\cite{EsfandiariKM+21} also give batched algorithms for stochastic linear bandits and adversarial multi-armed bandits.

Recently, \cite{SahaGi22} designed a fully adaptive algorithm  achieving an optimal regret of 
$\sum_{j: \Delta_{j} > 0} \frac{O( \log T)}{\Delta_{j}}$ for dueling bandits under the Condorcet setting. This algorithm is based on the idea of \emph{dueling} two classical bandit (MAB)   algorithms against each other in a repeated zero-sum game with carefully designed rewards. 
The reward for one algorithm depends on the actions of the other; hence,
these algorithms need to 
achieve \emph{best-of-both-worlds} guarantee for both stochastic and adversarial settings. 
However, the approach of \cite{SahaGi22}  is not directly applicable to the {\em batched} setting that we consider. This is because, as shown by \cite{EsfandiariKM+21}, any $B$-round algorithm for batched MAB in the adversarial setting has regret $\Omega({T/B})$.

Adaptivity and batch processing has been recently studied for stochastic submodular cover~\cite{GolovinK-arxiv, AAK19, EsfandiariKM19, GGN21}, and for various stochastic ``maximization'' problems such as knapsack~\cite{DGV08,BGK11}, matching~\cite{BGLMNR12,BehnezhadDH20}, probing~\cite{GN13} and orienteering~\cite{GuhaM09,GuptaKNR15,BansalN15}.
Recently, there have also been several results examining the role of adaptivity in (deterministic) submodular optimization; e.g. ~\cite{BalkanskiS18,BalkanskiBS18,BalkanskiS18b,BalkanskiRS19,ChekuriQ19}. 
\fi

\section{The Batched Algorithm}\label{sec:alg}
\ifConfVersion
\vspace{-0.1in}
\else
\fi
In this section, we describe a $B$-round algorithm for the $K$-armed dueling bandit problem under the Condorcet condition. Recall that given a set of $K$ arms, $\B = \{1, \ldots, K\}$, and a positive integer $B \leq \log(T)$, we wish to find a sequence of $B$ batches of noisy comparisons with low regret.
Given arms $i$ and $j$, recall that $p_{i, j} = \frac{1}{2} + \Delta_{i, j}$ denotes the probability of $i$ winning over $j$ where $\Delta_{i, j} \in \left(-1/2, {1}/{2}\right)$.  
We use $a^*$ to denote the Condorcet winner; recall that $a^*$ is a Condorcet winner if $p_{a^*, j} \geq 1/2$ for all $j \in \B$. To simplify notation, we use $\Delta_j = \Delta_{a^*,j}$.
Before describing our algorithm, we first define some notation.
We use $\A$ to denote the current set of \emph{active} arms; i.e., the arms that have not been eliminated. We will use index $r$ for rounds or batches. If pair $(i, j)$ is compared in round $r$, it is compared $q_r =  \lfloor q^r \rfloor $ times where $q = T^{1/B}$. We define the following quantities at the \emph{end} of each round $r$:
\begin{itemize}
\item $N_{i, j}(r)$ is the total number of times the pair $(i, j)$ has been compared.
\item $\widehat{p}_{i, j}({r})$ is the frequentist estimate of $p_{i, j}$, i.e.,  
\begin{align}
\label{eq:pair_est}
\widehat{p}_{i, j}({r}) = \frac{\# \ i\text{ wins against } j \text{ until end of round } r}{N_{i, j}(r)}
    \,.
\end{align}

\item  Two confidence-interval radii for each $(i,j)$ pair:
\begin{equation}
c_{i, j}(r) = \sqrt{\frac{2\log(2K^2q_r)}{N_{i, j}(r)}} \qquad \mbox{and}\qquad \gamma_{i, j}(r) = \sqrt{\frac{\log(K^2BT)}{2N_{i, j}(r)}}
\end{equation} 

\end{itemize}

We now describe our $B$-round algorithm, called \textsc{{\bf C}atching the {\bf C}ondorcet winner in {\bf B}atches} (or, \balg). At a high-level, the algorithm identifies the best arm $a^*$ in a small expected number of rounds, after which it uses this arm as an ``anchor'' to eliminate sub-optimal arms while incurring low regret. 
In every round $r$, we do the following:
\begin{enumerate}
    \item  We define a \emph{defeated set} $D_r(i)$ for every active arm $i$; this set comprises arms that are defeated \emph{with confidence} by $i$. Specifically, $j \in D_r(i)$ if $\widehat{p}_{i, j}(r-1) > 1/2 + c_{i, j}(r-1)$.
    
    \item Then, we define a \emph{candidate} $i_r$ as the arm that defeats the most number of arms; that is, $i_r = \arg\max_{i \in \A} |D_r(i)|$.
    
    \item For every arm $i \neq i_r$: 
    \begin{itemize}
        \item If $i \in D_r(i_r)$, then we compare $i$ to $i_r$ for $q_r$ times. The idea here is to use $i_r$ as an anchor against $i$. {We will show that $a^*$ becomes the candidate $i_r$ in a small number of rounds.
        Then, this step ensures that we eliminate arms efficiently using $a^*$ as an anchor.}
        \item If  $i \notin D_r(i_r)$, then $i$ is compared to all arms in $ \A$ for $q_r$ times. This step crucially protects the algorithm against cases where a sub-optimal arm becomes the candidate (and continues to become the candidate).
        {For example, suppose $K = [5]$ and the arms are linearly ordered as $1 \succ 2 \succ \cdots \succ 5$. Furthermore suppose that in some round $r$, we have that (a) $2$ defeats $3, 4, 5$ and (b) $1$ (best arm) defeats $2$ but not the others. So, $2$ is the candidate in round $r$; if $1$ is not compared to $3, 4, 5$, then $2$ would continue to be the candidate (leading to high regret).}
    \end{itemize}
    
    \item  If, for any arm $j$, there is arm $i$ 
    such that $\widehat{p}_{i,j}(r) > \frac{1}{2} + \gamma_{i, j}(r)$, then  $j$ is eliminated from $\A$.
    
\end{enumerate}

This continues until $T$ total comparisons are performed. See \Cref{alg:b-round-condorcet} for a formal description. 
The main result of this section is to show that \balg \ achieves the guarantees stated in Theorems~\ref{thm:condorcet-init-high-prob} and \ref{thm:condorcet-init-exp}.

\begin{algorithm}
\caption{\balg  \ (\textsc{Catching the Condorcet winner in Batches})}
\label{alg:b-round-condorcet}
\begin{algorithmic}[1]
\State \textbf{Input:} Arms $\B$, time-horizon $T$, integer $B \geq 1$
\State active arms $\A \gets \B$,  $r \gets 1$, emprical probabilities $\widehat{p}_{i, j}(0) = \frac{1}{2}$ for all $i, j\in \B^2$
\While{number of comparisons $\leq T$} 
\State \textbf{if} $\A = \{i\}$ for some $i$ \textbf{then} play $(i,i)$ for remaining trials
\State $D_r(i) \gets \{j \in \A : \widehat{p}_{i,j}(r-1) > \frac{1}{2} + c_{i, j}(r-1)\}$
\State $i_r \gets \arg\max_{i \in \A} |D_r(i)|$
\For{ $i \in \A \setminus \{i_r\}$}
\If{$i \in D_r(i_r)$}  
\State compare $(i_r, i)$ for $q_r$ times
\Else
\State for each $j \in \A$, compare $(i, j)$ for $q_r$ times  
\EndIf
\EndFor
\State compute $\widehat{p}_{i, j}(r)$ values
\If{$\exists i, j$ : $\widehat{p}_{i, j}(r) > \frac{1}{2} +\gamma_{i, j}(r)$} 
\State $\A \gets \A \setminus \{j\}$
\EndIf
\State $r \gets r+1$
\EndWhile
\end{algorithmic}
\end{algorithm}

\paragraph{Overview of the Analysis.} We provide a brief outline of the proofs of our main 
results. Let $\delta>0$ be any value. Towards proving \Cref{thm:condorcet-init-high-prob}, we first define two events:
\begin{itemize}
    \item The first event, denoted $G$, ensures that $a^*$ is not eliminated during the execution of \balg. We show that $\pr(G) \geq 1 - 1/T$.
    \item The second event, denoted $E(\delta)$, says that there exists a round $C(\delta)$ (defined later) such that for all $r > C(\delta)$, the estimate $\widehat{p}_{i, j}(r-1)$ satisfies the confidence interval of $c_{i, j}(r-1)$. Moreover, $\pr(E(\delta)) \geq 1 -\delta$.
\end{itemize}
By union bound, $\pr(G \cap E(\delta)) \geq 1 - \delta- 1/T$. Together, we use $G$ and $E(\delta)$ to argue that:
\begin{itemize}
    \item the best arm, $a^*$, is \emph{not defeated} by any arm $i$ in any round $r > C(\delta)$,
    \item and that there exists a round $\rdel\ge C(\delta)$ such that for every round after $\rdel$, arm $a^*$ defeats \emph{every other arm}.
\end{itemize}
Under the event $G \cap E(\delta)$, we analyze the regret in two parts: (i) regret incurred up to round $\rdel$, which is upper
bounded by $K^2 \sum_{r \leq \rdel} q^{r}$ and 
(ii) regret after $\rdel$, which is the regret 
incurred in eliminating sub-optimal arms using $a^*$
as an anchor.
Finally, we can use the high-probability bound to also obtain a bound on the expected regret, proving \Cref{thm:condorcet-init-exp}. 
\ifConfVersion
We provide some details of the proof of \Cref{thm:condorcet-init-high-prob}. We defer the proof of \Cref{thm:condorcet-init-exp} to \Cref{app:exp}.
\else\fi

\subsection{The Analysis}
\ifConfVersion
\vspace{-0.1in}
\else
\fi
In this section, we prove high-probability and expected regret bounds for \balg.
Recall that $q = T^{1/B}$, and that 
$q \geq 2$. 
We first prove the following lemma which will be used to prove that $a^*$ is never eliminated.

\begin{lemma}\label{lem:confidence}
For any batch $r \in [B]$, and for any pair $(i, j)$,  
we have $$ \pr\left(|\widehat{p}_{i, j}(r) - p_{i, j}| > \gamma_{i, j}(r) \right) \leq 2\eta, $$ where
$\eta = 1/K^2BT$.
\end{lemma}
\begin{proof}
Note that $\E[\widehat{p}_{i, j}(r)] = p_{i, j}$, and applying Hoeffding's inequality gives  
$$ \pr\left(|\widehat{p}_{i, j}(r) - p_{i, j}| >\gamma_{i, j}(r)\right) \leq 2\exp\left(-2N_{i, j}(r) \cdot \gamma_{i, j}(r)^2\right) \leq 2\eta. $$
\end{proof}

We first define the \emph{good} event $G$  as follows.
\begin{definition}[Event $G$]
An estimate $\widehat{p}_{i,j}(r)$ at the end of batch $r$ is {\bf \emph{strongly-correct}} if $|\widehat{p}_{i, j}(r) - p_{i, j}| \leq \gamma_{i, j}(r)$. 
We say that event $G$ occurs if every estimate in every batch $r\in [B]$ is strongly-correct.
\end{definition}

The following two lemmas show  that $G$ occurs with high probability and  that $a^*$ is not eliminated under $G$. 
\begin{lemma}\label{lem:correct-est-event}
The probability that every estimate in every batch of \balg \ is strongly-correct is at least $1 - 1/T$.
\end{lemma}
\begin{proof}
Applying \Cref{lem:confidence} and taking a union bound over all pairs and batches, we get 
that the probability that some estimate is not strongly-correct is at most $\binom{K}{2} \times B \times 2\eta \leq \frac{1}{T}$ where $\eta = 1/K^2BT$. Thus, $\pr(\overline{G}) \leq \frac{1}{T}$. 
\end{proof}

We now show that, under event $G$, the best arm $a^*$ is never eliminated.

\begin{lemma}\label{lem:a1}
Conditioned on $G$, the best arm $a^*$ is never eliminated from $\A$ in the elimination step of \balg.
\end{lemma}
\begin{proof}
In \balg, an arm $j$ is deleted in batch $r$ iff there is an arm $i\in \A$ with $\widehat{p}_{i, j}(r) > \frac{1}{2} + \gamma_{i, j}(r)$. If $a^*$ is eliminated due to some arm $j$, then by definition of event $G$,  we  get $p_{j, a^*}\ge \widehat{p}_{i, j}(r) - \gamma_{i, j}(r) > \frac12$, a contradiction.
\end{proof}

\subsubsection{High-probability Regret Bound}
We now prove Theorem~\ref{thm:condorcet-init-high-prob}. Fix any $\delta>0$. We first define another good event as follows.
\begin{definition}[Event $E(\delta)$]
An estimate $\widehat{p}_{i,j}(r)$ in batch $r$ is {\bf \emph{weakly-correct}} if $|\widehat{p}_{i, j}(r) - p_{i, j}| \leq c_{i, j}(r)$. Let $ C(\delta) := \lceil \frac{1}{2} \log_q(1/\delta)\rceil$.  
We say that event $E(\delta)$ occurs if for each batch $r\ge C(\delta)$, every estimate  is weakly-correct.
\end{definition}

The next lemma shows that $E(\delta)$ occurs with probability at least $1-\delta$.
\begin{lemma}\label{lem:c-delta-conf}
For all $\delta > 0$, we have 
 $$ \pr(\neg E(\delta)) \,\,= \,\, \pr\left( \exists r \ge C(\delta), i, j : |\widehat{p}_{i, j}(r) - p_{i, j}| > c_{i, j}(r) \right)  \,\,\leq  \,\,\delta.$$ 
\end{lemma}
\begin{proof}
For any pair $i,j$ of arms and round $r$, let $B_{i, j}(r)$ denote the event that $ |\widehat{p}_{i, j}(r) - p_{i, j}| > c_{i, j}(r)$. Note that $N_{ij}(r)\le \sum_{s=1}^r q_s\le 2q_r$. For any integer $n$, let $s_{ij}(n)$ denote the sample average of $n$ independent Bernoulli r.v.s with probability $p_{ij}$. By Hoeffding's bound, 
$$\pr[|s_{ij}(n) - p_{ij}| > c]\le 2e^{-2n c^2},\qquad \mbox{ for any }c\in [0,1].$$ 
We now bound
\begin{align*}
\pr[B_{ij}(r)] & \le \sum_{n=0}^{2q_r} \pr[B_{ij}(r) \,\wedge\, N_{ij}(r)=n ]\\
&\le \sum_{n=0}^{2q_r} \pr\left[|s_{ij}(n) - p_{ij}| > \sqrt{\frac{2\log(2K^2 q_r)}{n}}\right] \le \sum_{n=0}^{2q_r} 2\exp\left(-2n\cdot  \frac{2\log(2K^2 q_r)}{n}\right)\\
&\le 4q_r\cdot \frac{1}{(2K^2 q_r)^4}  \le \frac{1}{4K^2\cdot q_r^2}
\end{align*}

The second inequality uses the definition of $c_{ij}(r)$ when $N_{ij}(r)=n$.
The last inequality uses $K\ge 2$. 
Next, by a union bound over arms and rounds, we can write the desired probability as 
\begin{align}
\pr(\exists r \ge C(\delta), i, j : B_{i, j}(r)) & \leq \sum_{r \ge  C(\delta)}  \sum_{i < j}\pr(B_{i, j}(r)) \notag \\
& \leq \sum_{r \ge C(\delta)} {K\choose 2} \cdot \frac{1}{4K^2\cdot q_r^2} \le \sum_{r \ge C(\delta)} \frac{1}{8q_r^2}\notag  \\
&\leq \sum_{r \ge C(\delta)}  \frac{1}{2q^{2r}} = \frac{1}{2q^{2C(\delta)}} \cdot \left(1 + \frac{1}{q^2 } + \frac{1}{q^{4}} + \cdots \right) \leq \frac{1}{q^{2C(\delta)}} \le \delta \label{eq:r-union-bnd}
\end{align}
The second inequality above uses the bound on $\pr[B_{ij}(r)]$. 
The first inequality in \eqref{eq:r-union-bnd} uses $q_r=\lfloor q^r\rfloor \ge q^r-1\ge \frac{q^r}{2}$ as $q\ge 2$. The last inequality in \eqref{eq:r-union-bnd} uses the definition of $C(\delta)$.

The lemma now follows by the definition of event $\neg E(\delta)$ as  $\exists r \ge C(\delta), i, j : B_{i, j}(r)$. 
\end{proof}

We will analyze our algorithm under both events $G$ and $E(\delta)$. 
\emph{Conditioned on these}, we next show:
\begin{itemize}
    \item The best arm, $a^*$, is \emph{not defeated} by any arm $i$ in any round $r > C(\delta)$ (\Cref{lem:not-defeated}).
    \item Furthermore, there exists a round $\rdel\ge C(\delta)$ such that arm $a^*$ defeats \emph{every other arm}, in every round after $\rdel$  (\Cref{lem:trapped}).
\end{itemize}
Intuitively, these observations imply that our algorithm identifies the best arm after $\rdel$ rounds. Thus, beyond round $\rdel$, we only perform pairwise comparisons of the form $(a^*, i)$ for $i \neq a^*$: thus, $a^*$ is used as an anchor to eliminate sub-optimal arms. Note that event $G$ is required to ensure that $a^*$ is not eliminated (especially in rounds before $C(\delta)$ where the \Cref{lem:c-delta-conf} does not apply).  We now  prove the aforementioned observations.

\begin{lemma}\label{lem:not-defeated}
Conditioned on $G$ and $E(\delta)$, for any round $r > C(\delta)$, arm  $a^*$ is not defeated by any other arm, i.e.,  $a^* \notin \cup_{i\ne a^*} D_r(i)$.
\end{lemma}
\begin{proof} Fix any round $r\ge C(\delta)+1$. 
Suppose that $a^* \in D_r(i)$ for some other arm $i$. This implies that $\widehat{p}_{i, a^*}(r-1) > \frac{1}{2} + c_{i, a^*}(r-1)$. But under event $E(\delta)$, we have $|\widehat{p}_{i, a^*}(r-1) - p_{i, a^*}| \leq c_{i, a^*}(r-1)$ because $r-1\ge C(\delta)$. Combined, these two observations imply $p_{i, a^*} > \frac{1}{2}$, a contradiction.
\end{proof}

To proceed, we need the following definitions.
\begin{definition}
The candidate $i_r$  of round $r$ is called the {\bf \emph{champion}} 
if $|D_r(i_r)| = |\A| - 1$; that is, if  $i_r$ defeats every other active arm.
\end{definition}

\begin{definition}\label{def:r-star}
Let $\rdel\ge C(\delta)+1$ be the smallest integer such that 
\begin{equation*}
q^{\rdel} \ge  2A\log A,\qquad \mbox{where }A:=\frac{32}{\Delta_{\min}^2}\cdot \log(2 K^2).
\end{equation*}
\end{definition}
We use the following inequality based on this choice of $\rdel$.
\begin{lemma}\label{lem:rdel}
The above choice of $\rdel$ satisfies
$$ q^{r} > \frac{8}{\Delta_{\min}^2}\cdot \log\left(2 K^2 q_{r}\right),\qquad \forall r\ge \rdel . $$
\end{lemma}
\begin{proof}
Using the fact that $q_r\le q^r$, it suffices to show $q^r\ge \frac{8}{\Delta_{\min}^2}\cdot \left(\log(2 K^2) + \log q^r\right)$. 
Moreover,
$$\log(2 K^2) + \log q^r \le \left(1+\log(2 K^2) \right) \cdot \left(1+\log q^r \right) \le 4\cdot \log(2 K^2) \cdot \log q^r,$$
where the last inequality uses $K\ge 2$, $r\ge \rdel\ge 1$ and $q\ge 2$. So, it suffices to show:
\begin{equation}
\label{eq:rdel-calc}
q^{r} > A\cdot \log(q^r),\quad \forall r\ge \rdel,\qquad \mbox{where }A=\frac{32}{\Delta_{\min}^2}\cdot \log(2 K^2)
\end{equation}
Below, let $x=q^r$, $R:=2A\log A$ and function $f(x):=x-A\log x$. We will show that  $f(x) > 0$ for all $x\ge R$, which would imply \eqref{eq:rdel-calc} because $q^{\rdel}\ge R$. As $R\ge A$, and $f$ is increasing for $x\ge A$, it suffices to show that $f(R)\ge 0$. 
Indeed,
$$
\frac{f(R)}A = 2 \log A - \log(2A\log A) = \log A - \log(2\log A)  > 0,
$$
where the inequality uses $A\ge 8$.
\end{proof}

Then, we have the following.

\begin{lemma}\label{lem:trapped}
Conditioned on $G$ and $E(\delta)$, the best arm $a^*$ is the champion in every round $r>\rdel$. 
\end{lemma}
\begin{proof} 
We first argue that $a^*$ is compared to all active arms in each round $r\ge \rdel$. By Lemma~\ref{lem:a1}, we know $a^* \in \A$.  By \Cref{lem:not-defeated}, we have  $a^* \notin D_{r}(j)$ for all $j \neq a^*$ because $r\ge \rdel\ge 1+ C(\delta)$. 
If candidate $i_{r} \neq a^*$, then $a^*$ will be compared to all $j \in \A$ (since $a^* \notin D_{r}(i_{r})$). On the other hand, if $i_{r} = a^*$, then (1) for any $j \in D_{r}(a^*)$, arm $j$ is only compared to $a^*$, and (2) for any $j \in\A\setminus  D_{r}(a^*)$, arm $j$ is compared to all active arms including $a^*$.

Next, we show that for any round $r\ge \rdel+1$, arm $a^*$ defeats all other arms, i.e.,    $|D_{r}(a^*)| = |\A|-1$. This would imply that $i_r=a^*$ and $a^*$ is the champion. Consider any arm $j\in \A\setminus a^*$. 
Since $a^*$ is compared to all active arms in round $r-1\ge \rdel$, we have 
$$N_{a^*, j}(r-1) \geq q^{r-1} > \frac{8}{\Delta_{\min}^2} \cdot \log\left(2K^2 q_{r-1}\right),$$
where the second inequality is by Lemma~\ref{lem:rdel} with $r-1\ge \rdel$.  
Now, by definition,  we have 
$$ c_{a^*, j}(r-1) = \sqrt{\frac{2\log\left(2K^2  q_{r-1}\right)}{N_{a^*, j}(r-1)}} < \sqrt{\frac{2\log\left(2K^2  q_{r-1}\right)}{\frac{8}{\Delta_{\min}^2}  \log\left(2K^2 q_{r-1}\right)}} = \frac{\Delta_{\min}}{2}. $$
Given this, it is easy to show that $a^*$ defeats arm $j$ in round $r$.
 Conditioned on $E(\delta)$, we know that $|\widehat{p}_{a^*, j}(r-1) - p_{a^*, j}| \leq c_{a^*, j}(r-1) < \frac{\Delta_{\min}}{2}$. 
Then, we have $$ \widehat{p}_{a^*, j}(r-1) > p_{a^*, j} - \frac{\Delta_{\min}}{2} = \frac{1}{2} + \Delta_j - \frac{\Delta_{\min}}{2} \geq \frac{1}{2} + \frac{\Delta_{\min}}{2} > \frac{1}{2} + c_{a^*, j}(r-1).$$ 
Therefore, $j \in D_{r}(a^*)$.
It now follows that for any round $r\ge \rdel +1$, arm $a^*$ is the champion.
\end{proof}

We are now ready to prove \Cref{thm:condorcet-init-high-prob}.

\begin{proof}[Proof of \Cref{thm:condorcet-init-high-prob}]
First, recall that in round $r$ of \balg, any pair is compared $q_r =  \lfloor q^r \rfloor$ times where $q = T^{1/B}$. Since $q^B = T$, \balg \ uses at most $B$ rounds.

For the rest of proof, we fix $\delta > 0$. We now analyze the regret incurred by \balg, conditioned on events $G$ and $E(\delta)$. Recall that $\pr(G) \geq 1 - 1/T$ (\Cref{lem:correct-est-event}), and $\pr(E(\delta)) \geq 1 - \delta$ (\Cref{lem:c-delta-conf}). Thus, $\pr(G\cap E(\delta)) \geq 1 - \delta - 1/T$. Let $R_1$ and $R_2$ denote the regret incurred before and after round $\rdel$ (see \Cref{def:r-star}) respectively. 

\paragraph{Bounding $R_1$.} This is  the regret incurred up to (and including) round $\rdel$. We upper bound the regret by considering all pairwise comparisons every round $r\le \rdel$. 
\begin{align}
    R_1 \ \ &\leq \ \ K^2 \cdot \sum_{r \leq \rdel} q_r \leq \ \ K^2 \cdot \sum_{r \leq \rdel} q^r \ \ \leq \ \ 2K^2 \cdot q^{\rdel}  \notag \\
        &\leq 2 K^2   \cdot \max\left\{ q\cdot 2A\log A \, ,\, q^{C(\delta)+1} \right\}\notag,
\end{align}
where the last inequality uses Definition~\ref{def:r-star}; recall $A=\frac{32}{\Delta_{\min}^2}\cdot \log(2 K^2)$.
Plugging in the value of $C(\delta)\le 1+\frac{1}{2}\log_q(1/\delta)$, we obtain
\begin{equation}
 R_1 \le O(K^2)\cdot \max\left\{ q\cdot \frac{\log K}{\Delta_{\min}^2}\cdot \log\left(\frac{\log K}{\Delta_{\min}}\right) \, ,\, q^2\sqrt{\frac{1}{\delta}}\right\}.\label{eq:r1}  
\end{equation}

\paragraph{Bounding $R_2$.} This is the regret in rounds $r\ge \rdel+1$. By Lemma~\ref{lem:trapped}, arm $a^*$ is the champion in all these rounds. So, the only comparisons in these rounds are of the form $(a^*,j)$ for $j\in \A$. 

Consider any arm $j\ne a^*$. Let $T_j$ be  the total number of comparisons that $j$ participates in after round $\rdel$. Let $r$ be the penultimate round that $j$ is played in. We can assume that $r\ge \rdel$ (otherwise arm $j$ will never participate in rounds after $\rdel$, i.e., $T_j=0$). 
As arm $j$ is {\em not} eliminated after round $r$,  
$$\widehat{p}_{a^*, j}(r) \leq \frac{1}{2} + \gamma_{a^*, j}(r).$$
Moreover, by $E(\delta)$, we have $\widehat{p}_{a^*, j}(r) \ge p_{a^*, j} - c_{a^*, j}(r)$ because $r\ge \rdel\ge C(\delta)$.  So,
$$\frac12+\Delta_j = p_{a^*, j} \le \widehat{p}_{a^*, j}(r) + c_{a^*, j}(r)\le \frac{1}{2} + \gamma_{a^*, j}(r) +c_{a^*, j}(r).$$

It follows that 
$$ \Delta_j \leq \gamma_{a^*, j}(r) +c_{a^*, j}(r)  \le \frac{3}{\sqrt{2}}  \sqrt{\frac{ \log(2K^2B  T)}{N_{a^*, j}(r)}} $$ where the final inequality follows by definition of $c$ and $\gamma$. 
On re-arranging, we get $ N_{a^*, j}(r) \leq \frac{9\log(2K^2 B T)}{2\Delta_j^2}$. As $r+1$ is the last round that $j$ is played in, and $j$ is only compared to $a^*$ in each round after $\rdel$, 
$$ T_j \ \leq \ N_{a^*, j}(r+1) \ \leq \  N_{a^*, j}(r) + 2q\cdot N_{a^*, j}(r) \ \leq \  \frac{15q\cdot\log(2K^2 B T)}{\Delta_j^2}.$$
The second inequality follows since $j$ is compared to $a^*$ in rounds $r$ and $r+1$, and the number of comparisons in round $r+1$ is $\lfloor q^{r+1} \rfloor \leq q \cdot (2 q_r) \leq 2q \cdot N_{a^*, j}(r)$.
Adding over all arms $j$, the total regret accumulated beyond round $\rdel$ is 
\begin{equation}\label{eq:r2}
R_2 = \sum_{j \neq a^*} T_j \Delta_j \leq  \sum_{j \neq a^*} O\left(\frac{q\cdot \log(K T)}{\Delta_j}\right). 
\end{equation}

\noindent Combining \eqref{eq:r1} and \eqref{eq:r2}, and using $q=T^{1/B}$, we obtain 
\begin{align*}
R(T) &\leq O\left( T^{1/B} \cdot \frac{K^2 \log(K)}{\Delta_{\min}^2} \cdot \log\left(\frac{\log K}{\Delta_{\min}}\right) \right) \,+\, O\left( T^{2/B} \cdot K^2\cdot \sqrt{\frac{1}{\delta}}\right) \,+\, \sum_{j \neq a^*} O\left(\frac{T^{1/B}\cdot \log(K T)}{\Delta_j}\right). 
\end{align*}
 This completes the proof \Cref{thm:condorcet-init-high-prob}.
\end{proof}

\ifConfVersion
\else
\subsubsection{Expected Regret Bound}
In this section, we present the proof of \Cref{thm:condorcet-init-exp}.
We first state the definitions needed in the proof.
Let $F_X(\cdot)$ denote the cumulative density function (CDF) of a random variable $X$; that is, $F_X(x) = \pr(X \leq x)$. The inverse CDF of $X$, denoted $F^{-1}_X$, is defined as $F^{-1}_X(z) = \inf \{x : \pr(X \leq x) \geq z\}$ where $z \in [0, 1]$. We will use the identity $\E[X] = \int_0^1 F^{-1}_X(z) dz$. 

\begin{proof}[Proof of \Cref{thm:condorcet-init-exp}.]
First, note that in round $r$ of \balg, any pair is compared $q_r =  \lfloor q^r \rfloor$ times where $q = T^{1/B}$. Since $q^B = T$, \balg \ uses at most $B$ rounds.

Let $R(T)$ be the random variable denoting the regret incurred by \balg. By \Cref{thm:condorcet-init-high-prob}, we know that with probability at least $1 - \delta - 1/T$,
\begin{align*}
R(T) &\leq O\left( T^{1/B} \cdot \frac{K^2 \log(K)}{\Delta_{\min}^2} \cdot \log\left(\frac{\log K}{\Delta_{\min}}\right) \right) \,+\, O\left( T^{2/B} \cdot K^2\cdot \sqrt{\frac{1}{\delta}}\right) \,+\, \sum_{j \neq a^*} O\left(\frac{T^{1/B}\cdot \log(K T)}{\Delta_j}\right). 
\end{align*}
Thus, $F_{R(T)}^{-1}(1-\delta-1/T) \leq G(\delta)$ where 
$$ G(\delta) :=  A + O\left(T^{2/B} \cdot K^2\cdot \sqrt{\frac{1}{\delta}}\right) + B$$ where to simplify notation we set $A = O\left( T^{1/B} \cdot \frac{K^2 \log(K)}{\Delta_{\min}^2} \cdot \log\left(\frac{\log K}{\Delta_{\min}}\right) \right)$ and $B = \sum_{j \neq a^*} O\left(\frac{T^{1/B}\cdot \log(K T)}{\Delta_j}\right)$. 
Using the identity for expectation of a random variable, we get 
\begin{align*}
\E[R(T)] &= \int_0^1 F^{-1}_{R(T)}(z) dz \\
         &= \int_0^{1-\frac1T} F^{-1}_{R(T)}(z) dz +  \underbrace{\int_{1-\frac1T}^T F^{-1}_{R(T)}(z) dz}_{\leq T \cdot \frac{1}{T} = 1}\\
        &\leq \int_0^{1-\frac1T} F^{-1}_{R(T)}(z) dz \ \ + \ \ 1 \\
        &= \int_{1-\frac1T}^0 F^{-1}_{R(T)}\left(1-\delta-\frac{1}{T}\right) (-d\delta) \ \ + \ \ 1 \\
         &\leq \int_0^{1-\frac1T} G(\delta) d\delta \ \ + \ \ 1 \\
         &\leq A + O\left(T^{2/B}\cdot K^2\right) + B  +  1 
\end{align*}
where the fourth equality follows by setting $1-q-1/T = \delta$ and the final inequality follows since 
$\int_0^1 \left(\frac{1}{\delta}\right)^{1/2} \leq 2$. Thus, 
\begin{align*}
\E[R(T)] &\leq O\left( T^{1/B} \cdot \frac{K^2 \log(K)}{\Delta_{\min}^2} \cdot \log\left(\frac{\log K}{\Delta_{\min}}\right) \right) \,+\, O\left( T^{2/B} \cdot K^2\right) \,+\, \sum_{j \neq a^*} O\left(\frac{T^{1/B}\cdot \log(K T)}{\Delta_j}\right). 
\end{align*}
This completes the proof of \Cref{thm:condorcet-init-exp}.
\end{proof}
\fi

\ifConfVersion
\vspace{-0.1in}
\else
\fi
\section{Computational Results}\label{sec:computations}
\ifConfVersion
\vspace{-0.1in}
\else
\fi
\ifConfVersion
\begin{figure}[t]
     \centering
     \begin{subfigure}[b]{0.32\textwidth}
         \centering
         \includegraphics[width=\textwidth]{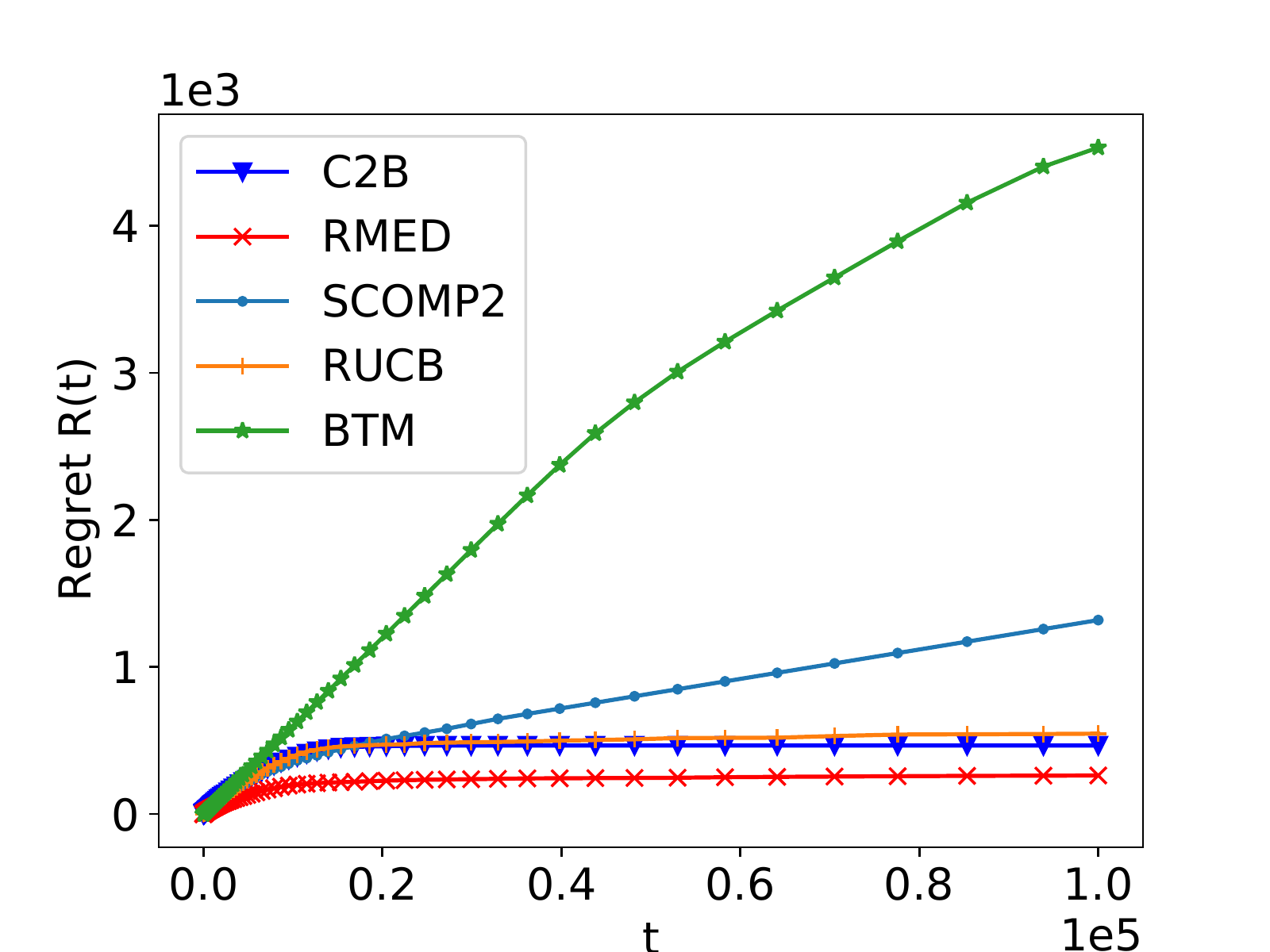}
         \caption{Six rankers}
     \end{subfigure}
     \begin{subfigure}[b]{0.32\textwidth}
         \centering
         \includegraphics[width=\textwidth]{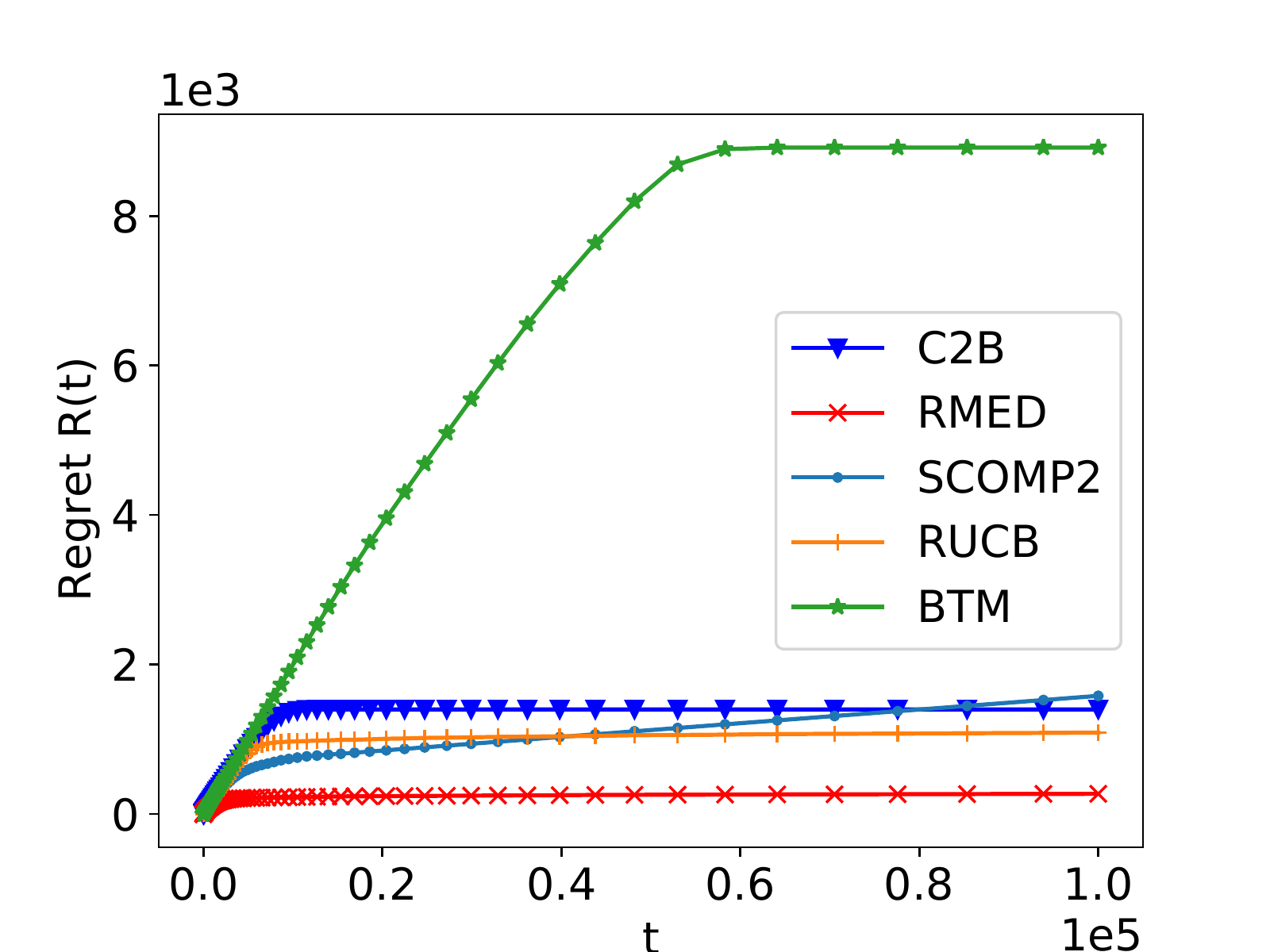}
         \caption{Sushi}
     \end{subfigure}
     \begin{subfigure}[b]{0.32\textwidth}
         \centering
         \includegraphics[width=\textwidth]{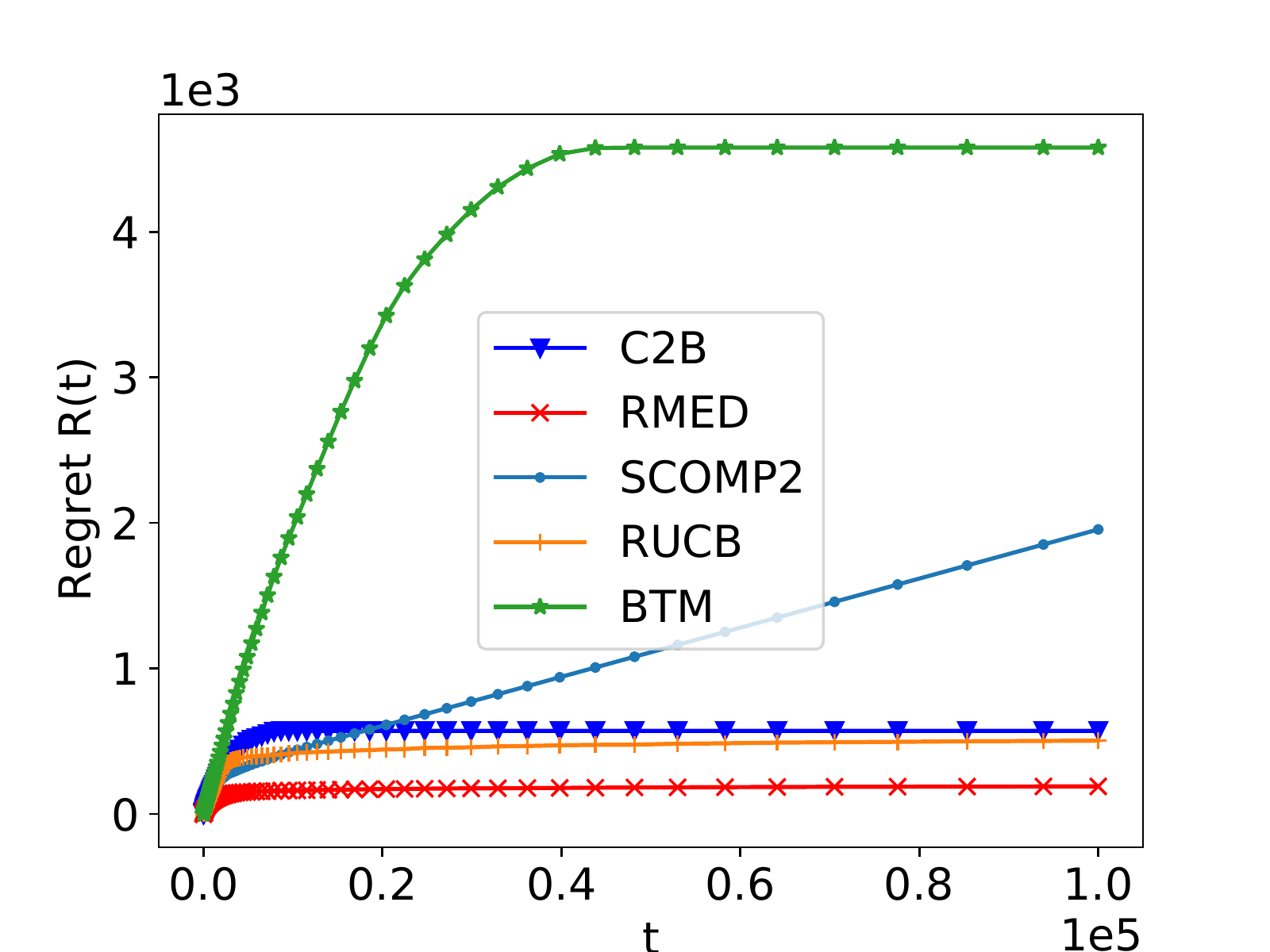}
         \caption{Irish-Meath}
     \end{subfigure}
     \begin{subfigure}[b]{0.32\textwidth}
         \centering
         \includegraphics[width=\textwidth]{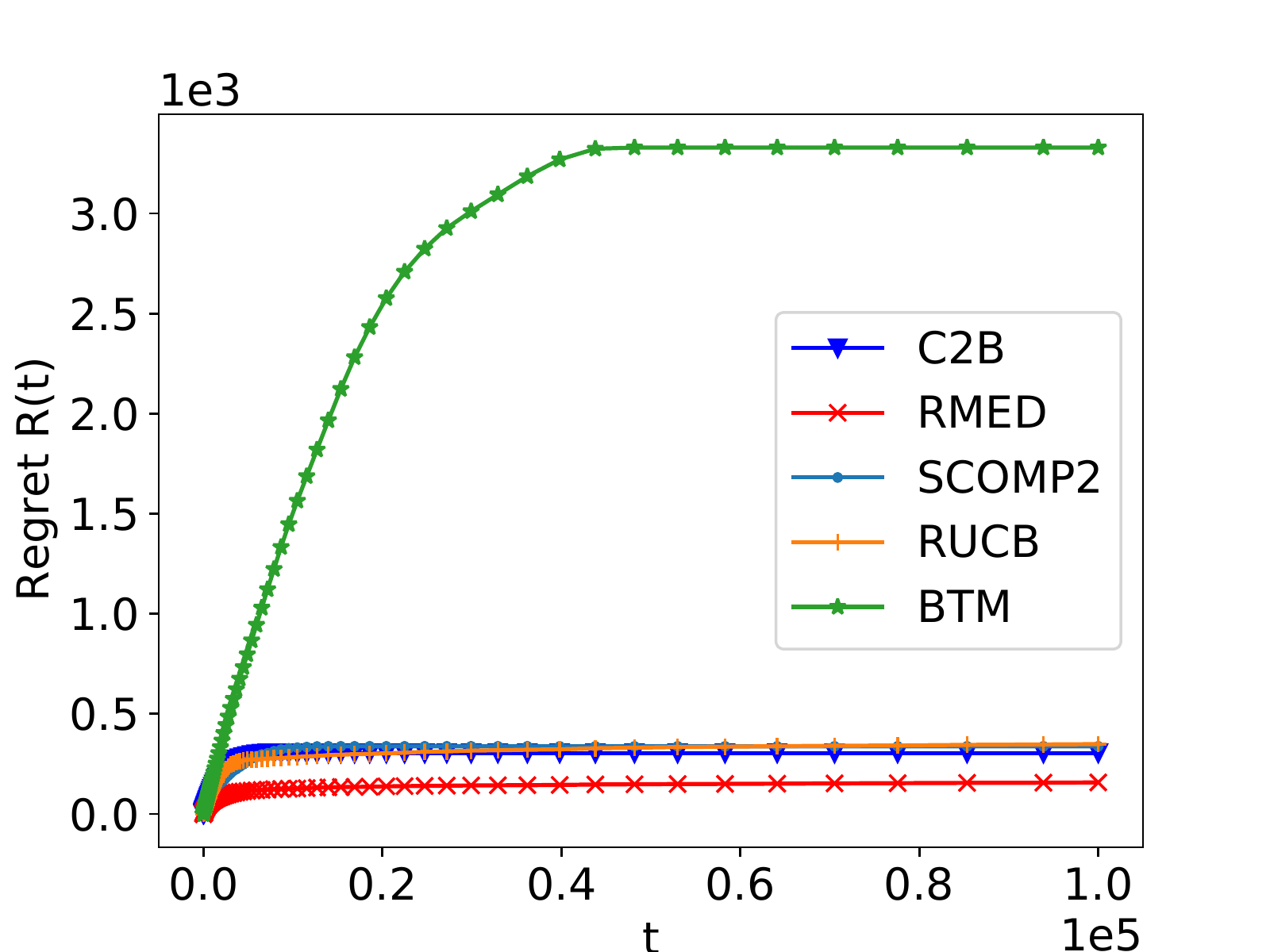}
         \caption{Irish-Dublin}
     \end{subfigure}
     \begin{subfigure}[b]{0.32\textwidth}
         \centering
         \includegraphics[width=\textwidth]{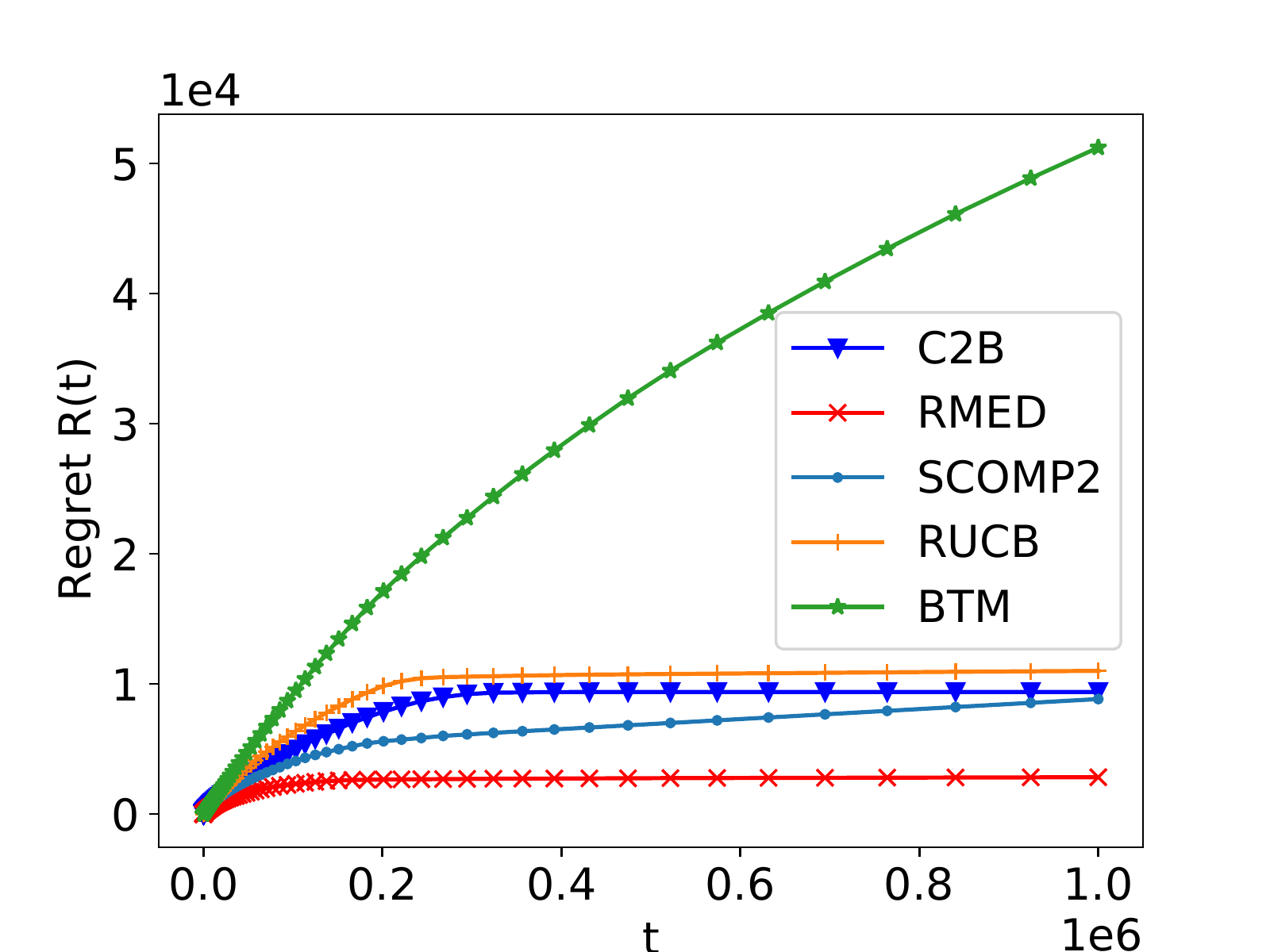}
         \caption{MSLR30}
     \end{subfigure}
     \begin{subfigure}[b]{0.32\textwidth}
         \centering
         \includegraphics[width=\textwidth]{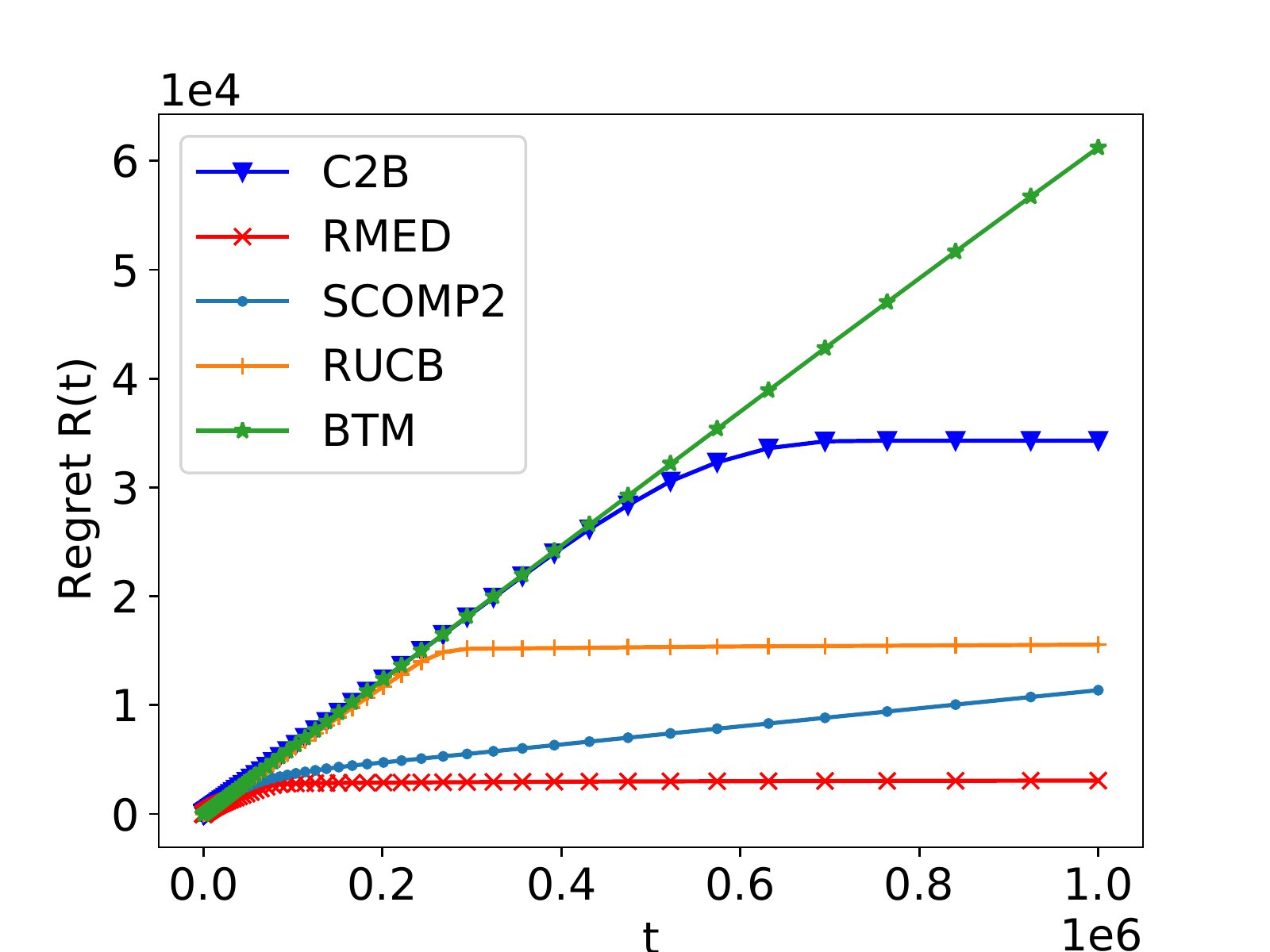}
         \caption{Yahoo30}
     \end{subfigure}
        \caption{Regret v/s t plots of algorithms when $B = \lfloor\log(T)\rfloor$}
        \label{fig:comparisons}
\end{figure}
\else
\fi

In this section, we provide details of 
our computational experiments.
The goal of our experiments is to answer the following questions: (i) How does the regret of \balg \ using $B = \lfloor \log(T)\rfloor$ batches compare to that of existing fully sequential as well as batched algorithms? and (ii) Can the regret of \balg \ match the regret of the best known sequential algorithms; if yes, then how many rounds suffice to achieve this?
Towards answering (i),
we compare \balg \ to
a representative set of sequential algorithms for dueling bandits using the library due to \cite{Komiyama+15a}. We compare \balg \ to the sequential algorithms 
RUCB~\cite{Zoghi+14}, RMED~\cite{Komiyama+15a}, and \textsc{Beat-the-Mean}\ (BTM)~\cite{YueJo11}. The reason that we chose these sequential algorithms is that our batched algorithm (\balg) is based on a similar paradigm, and such a  comparison demonstrates the power of adaptivity in this context.  
We also compare \balg \ to the batched algorithm SCOMP2 \cite{AgarwalGN22}.
We plot the cumulative regret $R(t)$ incurred by the algorithms against time $t$. 
We set $B = \lfloor \log(T)\rfloor$ for \balg \ and SCOMP2 in this experiment.
For (ii), we increased $B$ by a small amount; we found that the performance of \balg \ improves noticeably when given a constant number of additional rounds (we use $B = \lfloor \log(T)\rfloor + 6$ in this experiment).
We perform these experiments 
using the following real-world datasets.

\textbf{Six rankers.} This dataset is based on the $6$ retrieval functions used in the engine of ArXiv.org.

\textbf{Sushi.} The Sushi dataset is based on the Sushi preference dataset~\cite{Kamishima03} that contains the preference data regarding $100$ types of Sushi. A preference dataset using the top-$16$ most popular types of sushi is obtained.

\textbf{Irish election data.} The Irish election data for Dublin and Meath is available at \emph{preflib.org}. It contains partial preference orders over candidates. As in \cite{Agarwal+20}, these are transformed into preference matrices by selecting a subset of candiates to ensure that a Condorcet winner exists. There are $12$ candidates in the {\bf Irish-Meath} dataset, and $8$ in the {\bf Irish-Dublin} dataset.  

\textbf{MSLR and Yahoo!\ data.}  
We also run experiments on two web search ranking datasets: the Microsoft Learning to Rank (MSLR) dataset \cite{QinLiu13} and the Yahoo!\ Learning to Rank Challenge Set 1 \cite{ChapelleChang10}. These datasets have been used in prior work on online ranker evaluation \cite{Zoghi+15a, ChangLM+20}.
We use preference matrices generated using the ``navigational'' configuration 
(see \cite{ChangLM+20} for details). 
The MSLR dataset has $136$ rankers and the Yahoo! dataset has $700$ rankers. We sample $30$ rankers from each dataset while ensuring the existence of a Condorcet winner. In this way, we obtain two datasets, denoted {\bf MSLR30} and {\bf Yahoo30}.

Note that there exists a Condorcet winner in all datasets. 
We repeat each experiment $20$ times and report the average regret. 
In our algorithm, we use the \emph{KL-divergence based confidence bound} due to \cite{Komiyama+15a}  for elimination as it performs much better empirically, and our  theoretical bounds continue to hold (see \S\ref{sec:kl-alg}). 
This KL-divergence based 
elimination criterion eliminates an arm $i$ in round $r$ if $I_i(r) - I^*(r) > \log(T) +  f(K)$
where $I_i(r) = \sum_{j: \widehat{p}_{i,j}(r) < \half} N_{i,j}(r)  \cdot D_{\text{KL}}(\widehat{p}_{i,j}(r), \half)$
and $I^*(r) = \min_{j \in [K]} I_i(r)$.

\ifConfVersion
\else
\begin{figure}
     \centering
     \begin{subfigure}[b]{0.45\textwidth}
         \centering
         \includegraphics[width=\textwidth]{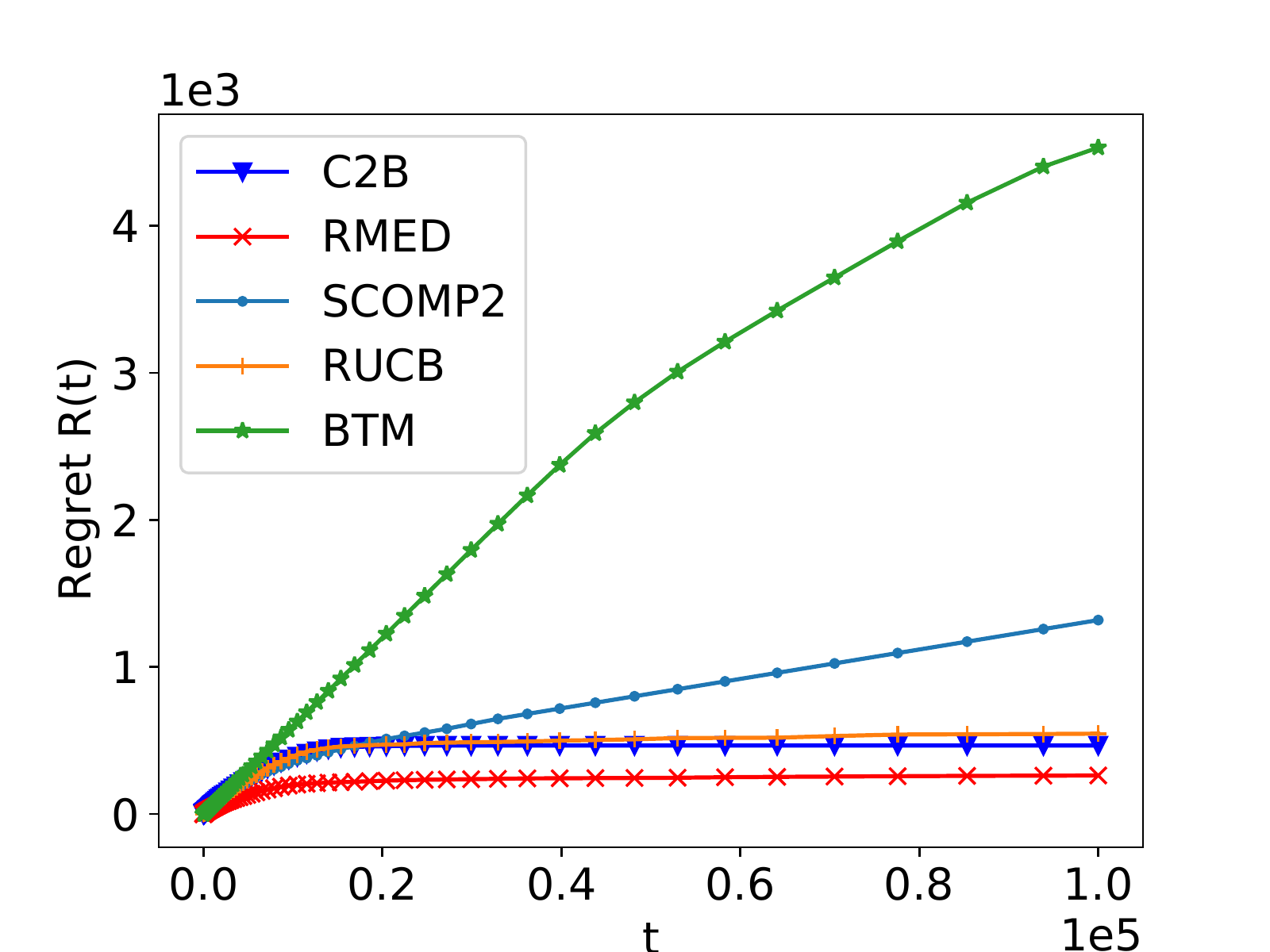}
         \caption{Six rankers}
     \end{subfigure}
     \begin{subfigure}[b]{0.45\textwidth}
         \centering
         \includegraphics[width=\textwidth]{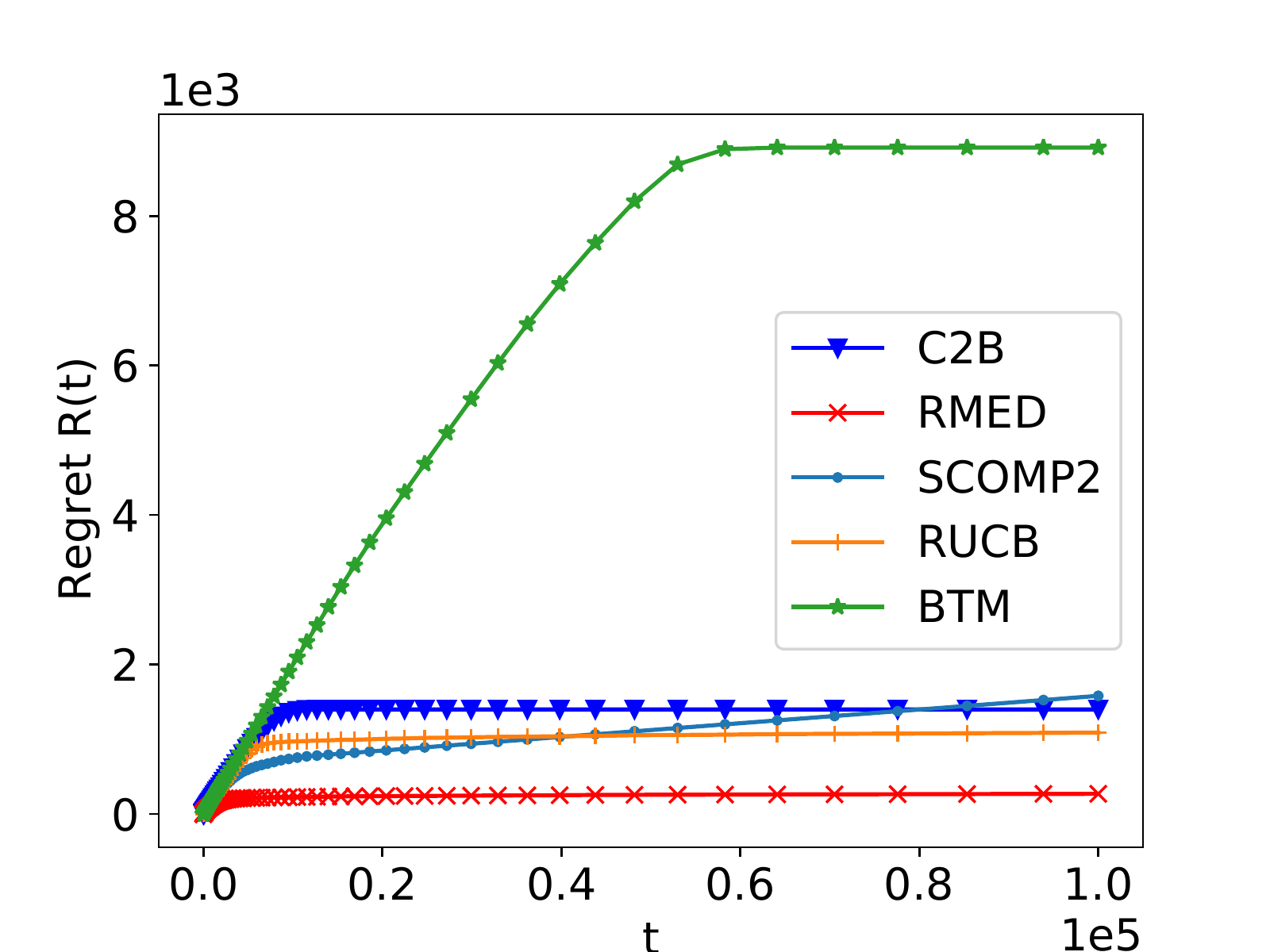}
         \caption{Sushi}
     \end{subfigure}
     \begin{subfigure}[b]{0.45\textwidth}
         \centering
         \includegraphics[width=\textwidth]{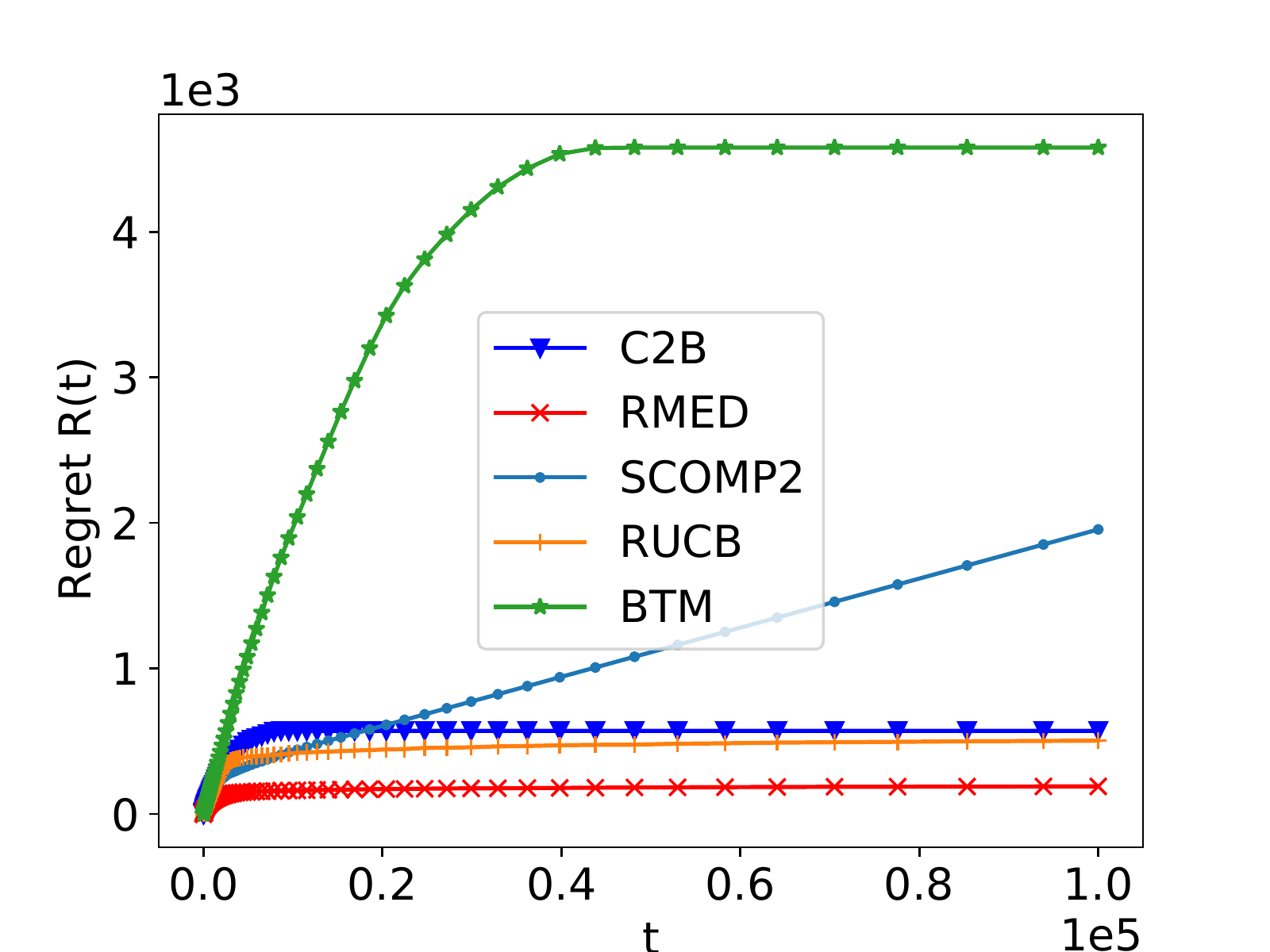}
         \caption{Irish-Meath}
     \end{subfigure}
     \begin{subfigure}[b]{0.45\textwidth}
         \centering
         \includegraphics[width=\textwidth]{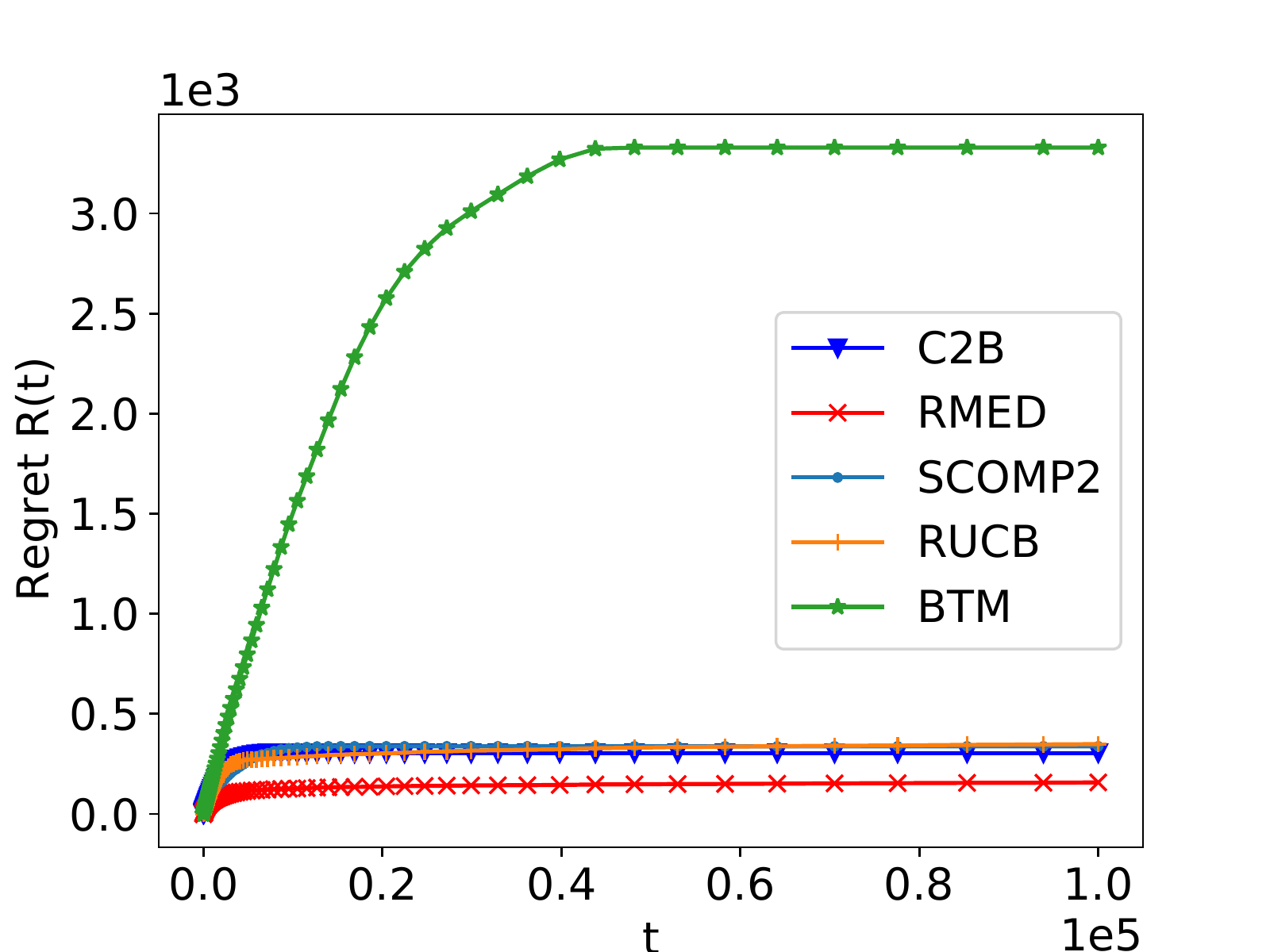}
         \caption{Irish-Dublin}
     \end{subfigure}
     \begin{subfigure}[b]{0.45\textwidth}
         \centering
         \includegraphics[width=\textwidth]{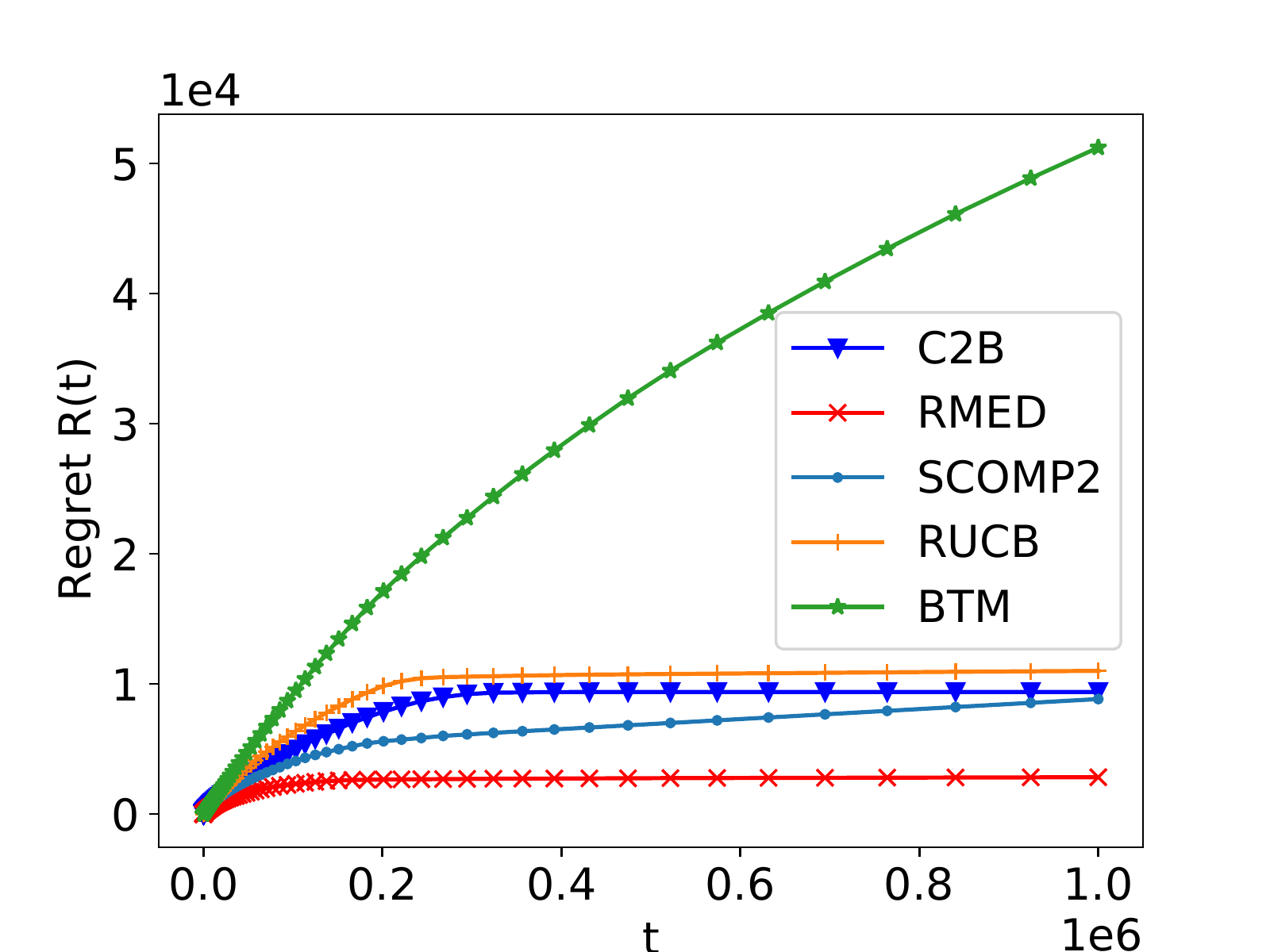}
         \caption{MSLR30}
     \end{subfigure}
     \begin{subfigure}[b]{0.45\textwidth}
         \centering
         \includegraphics[width=\textwidth]{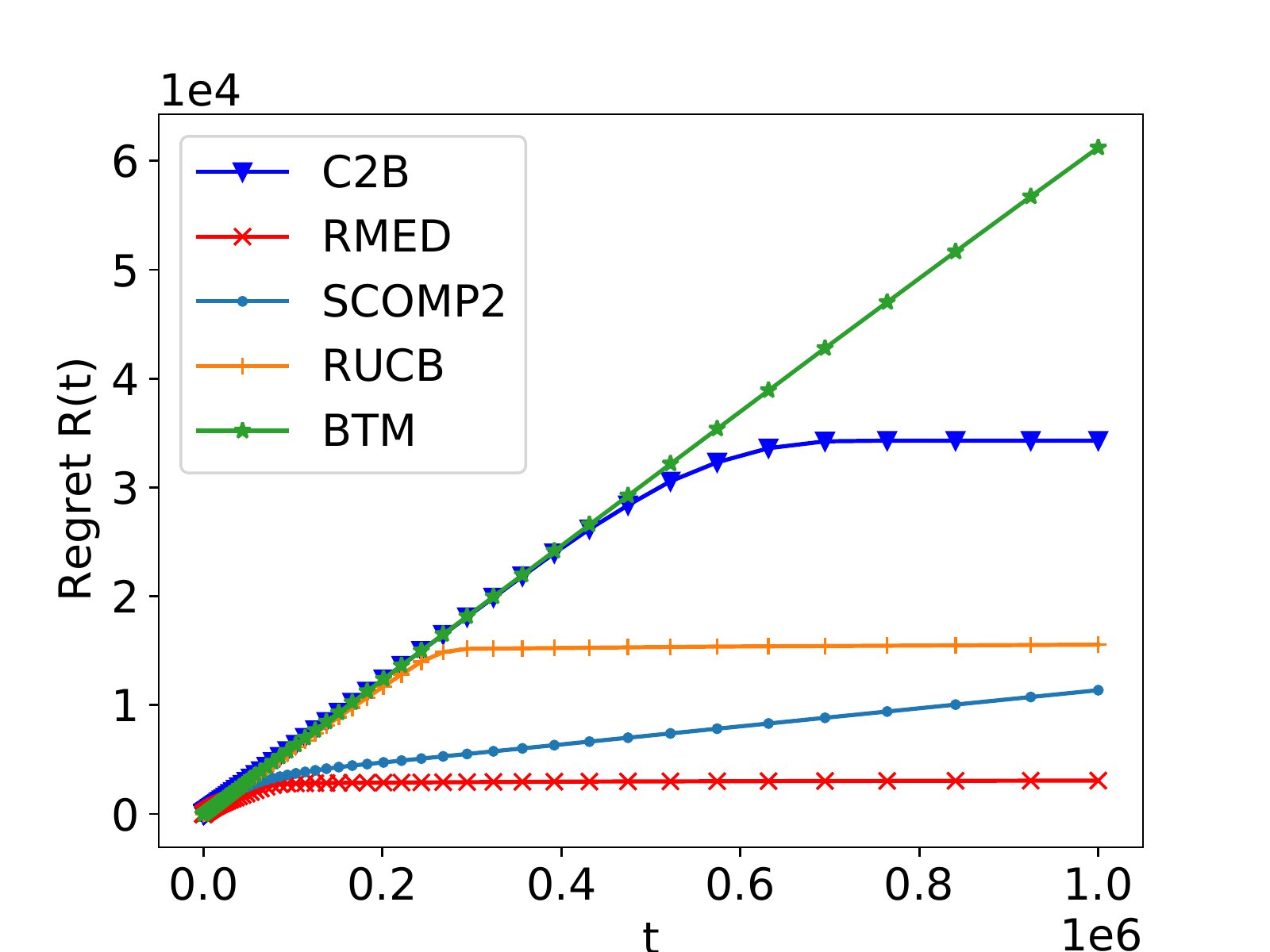}
         \caption{Yahoo30}
     \end{subfigure}
        \caption{Regret v/s t plots of algorithms when $B = \lfloor\log(T)\rfloor$}
        \label{fig:comparisons}
\end{figure}
\fi

\ifConfVersion
\begin{figure}[t]
     \centering
     \begin{subfigure}[b]{0.32\textwidth}
         \centering
         \includegraphics[width=\textwidth]{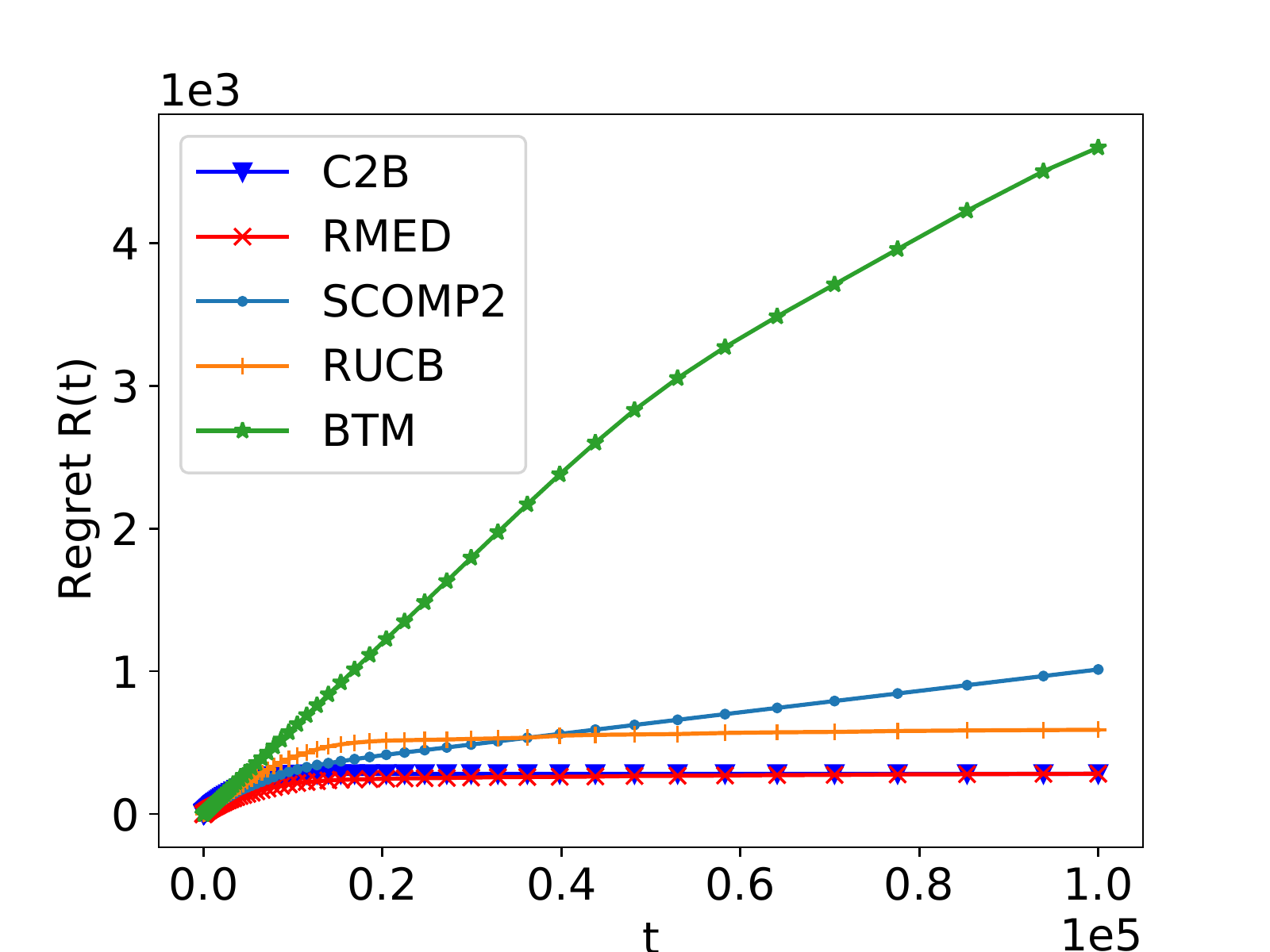}
         \caption{Six rankers}
     \end{subfigure}
     \begin{subfigure}[b]{0.32\textwidth}
         \centering
         \includegraphics[width=\textwidth]{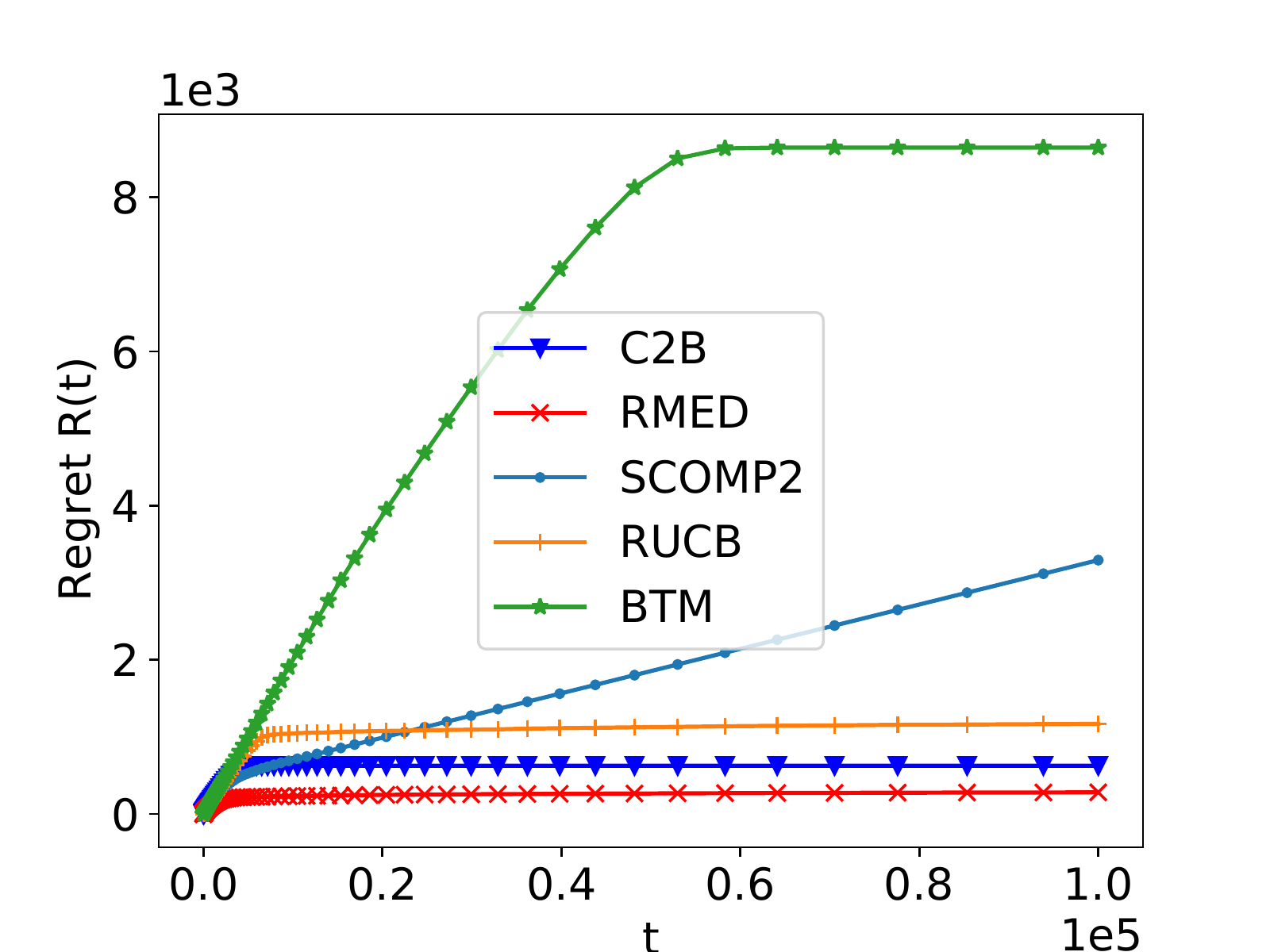}
         \caption{Sushi}
     \end{subfigure}
     \begin{subfigure}[b]{0.32\textwidth}
         \centering
         \includegraphics[width=\textwidth]{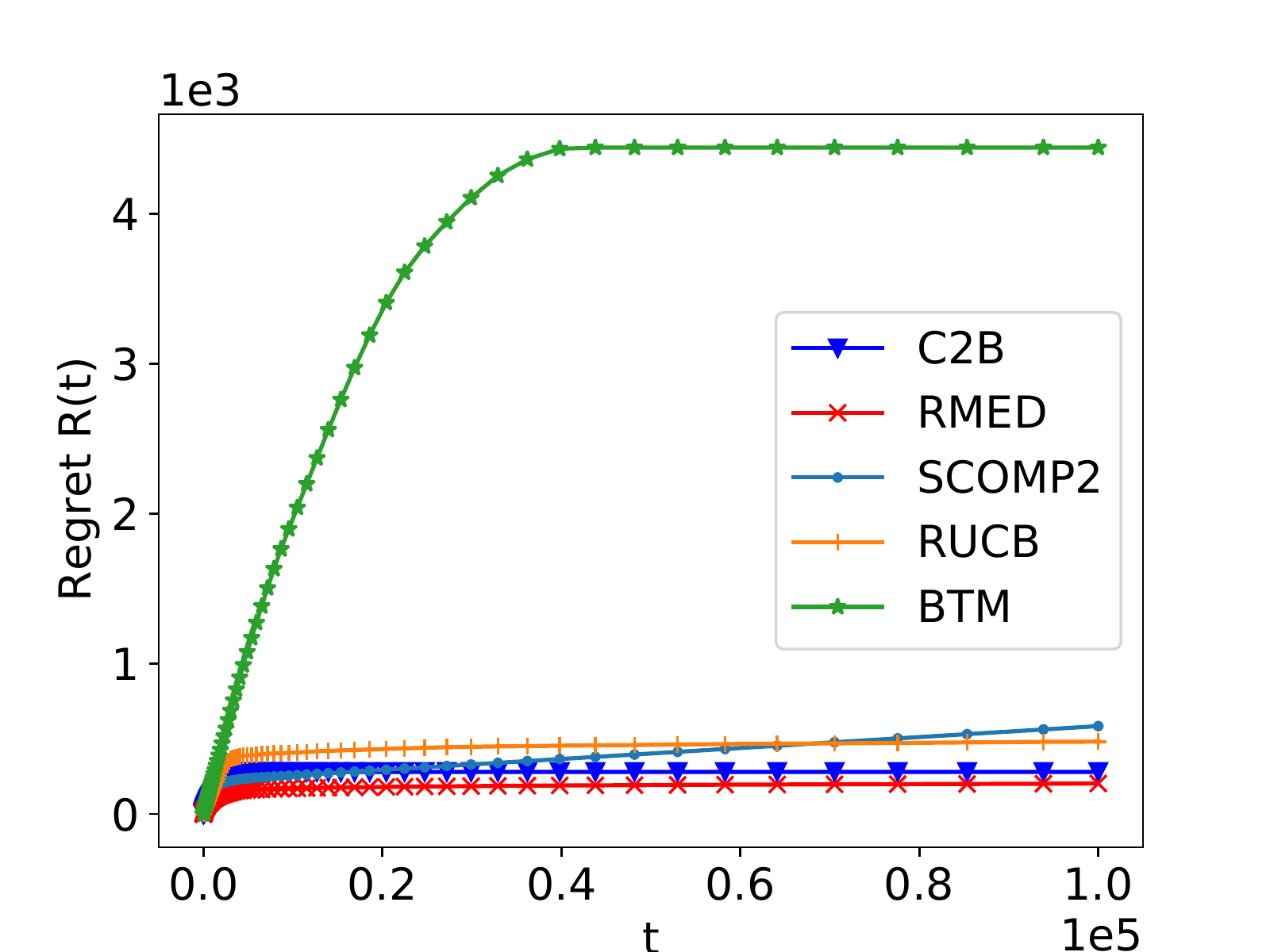}
         \caption{Irish-Meath}
     \end{subfigure}
     \begin{subfigure}[b]{0.32\textwidth}
         \centering
         \includegraphics[width=\textwidth]{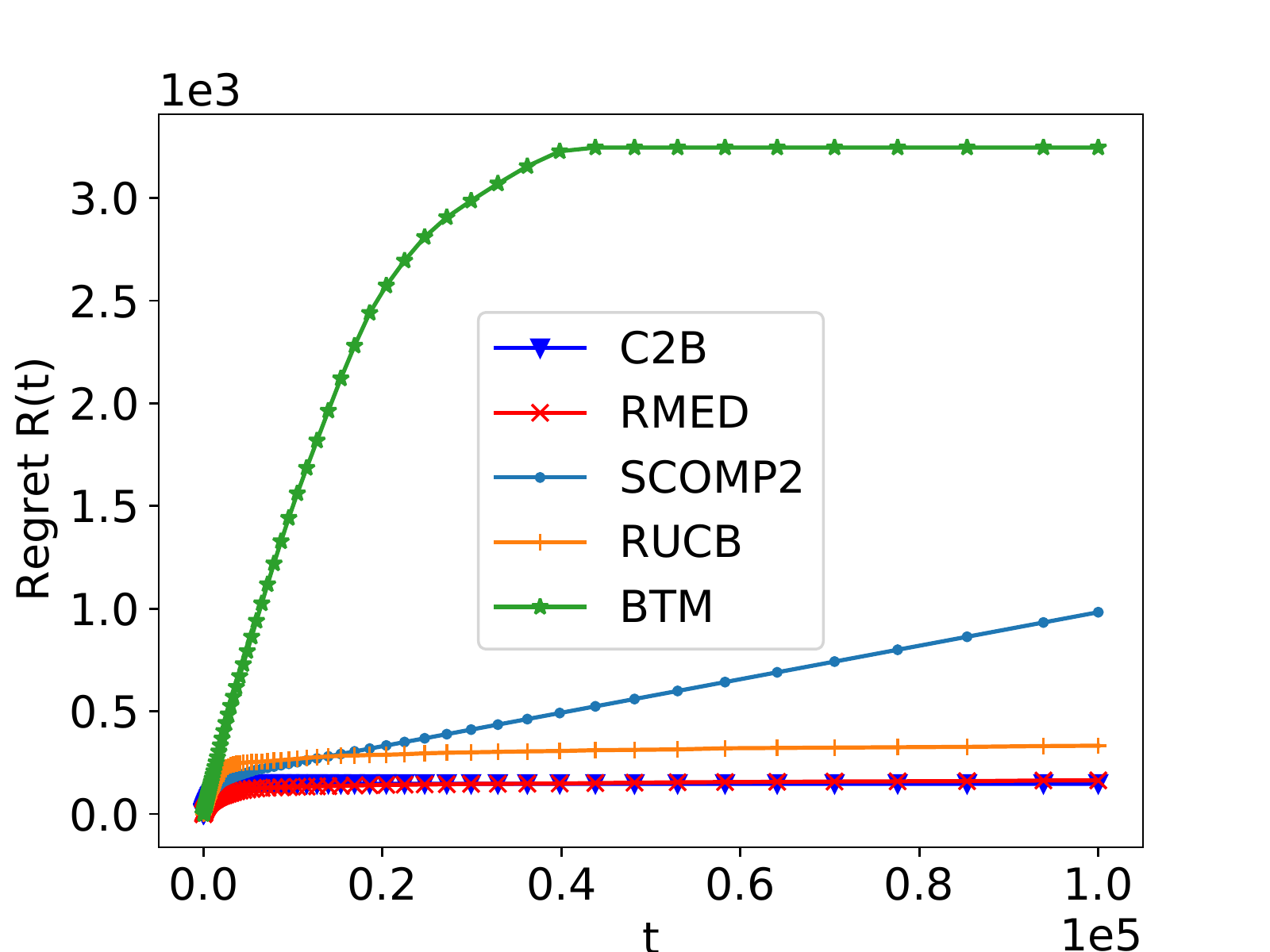}
         \caption{Irish-Dublin}
     \end{subfigure}
     \begin{subfigure}[b]{0.32\textwidth}
         \centering
         \includegraphics[width=\textwidth]{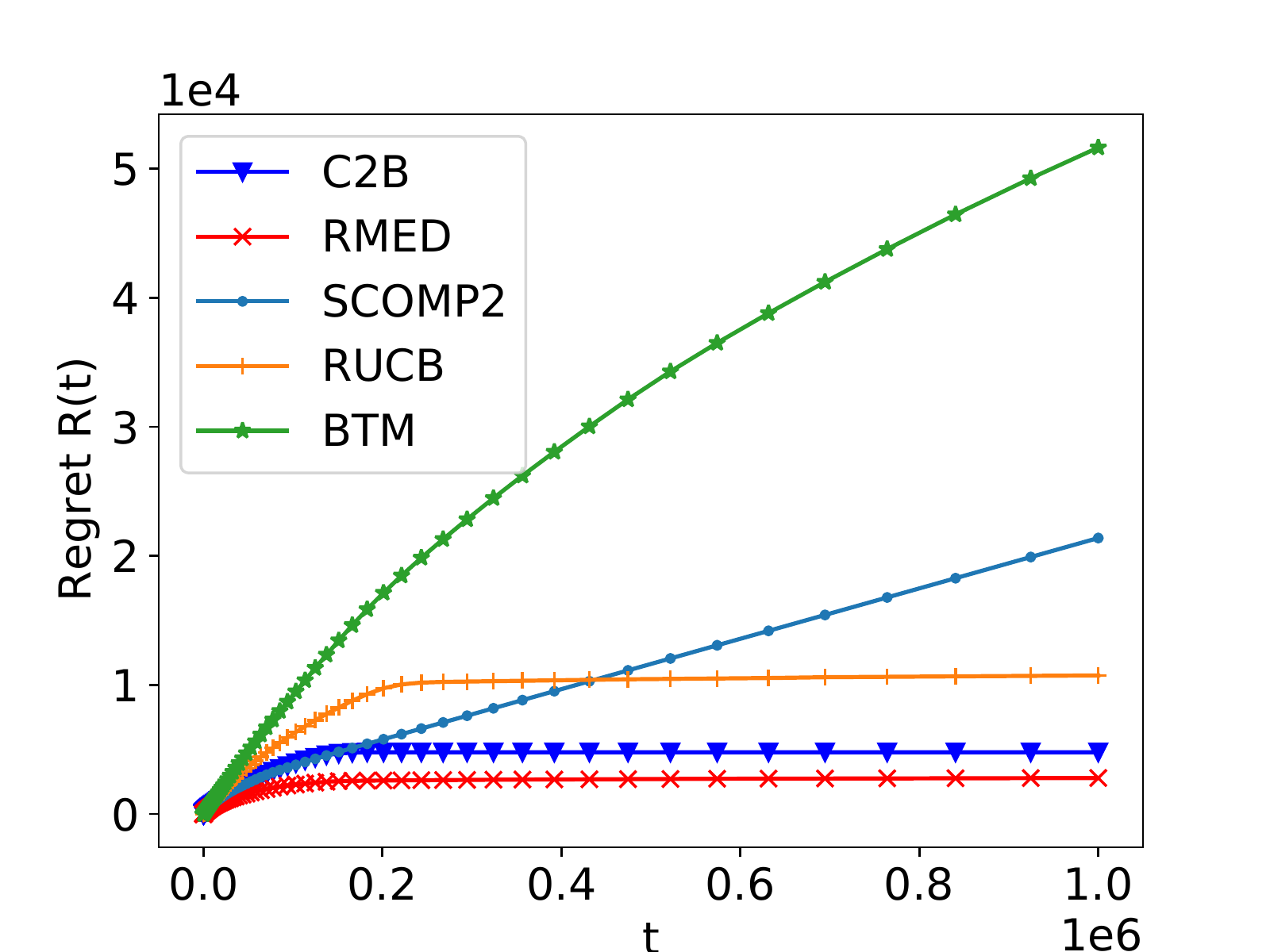}
         \caption{MSLR30}
     \end{subfigure}
     \begin{subfigure}[b]{0.32\textwidth}
         \centering
         \includegraphics[width=\textwidth]{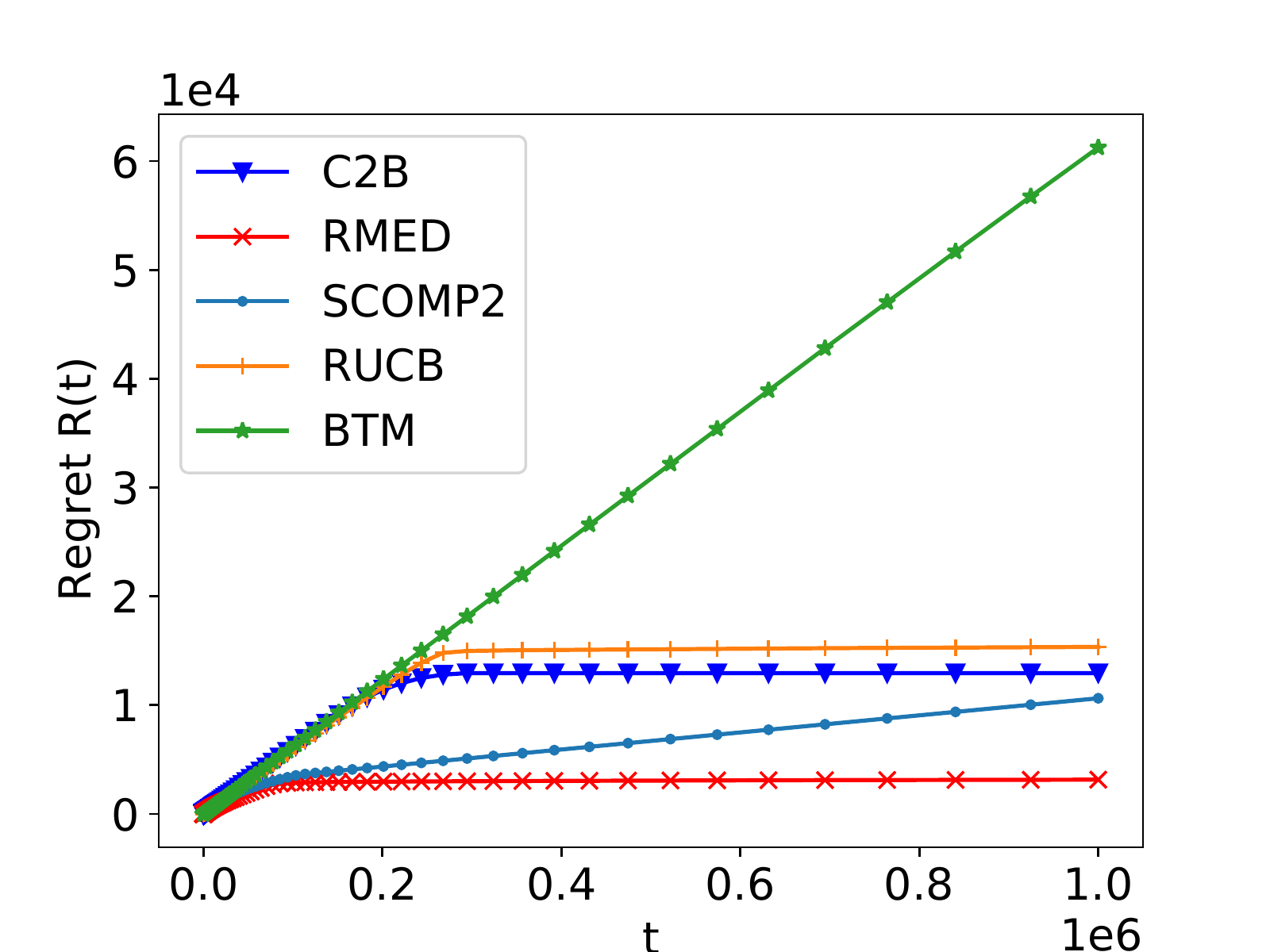}
         \caption{Yahoo30}
     \end{subfigure}
        \caption{Regret v/s t plots of algorithms when $B = \lfloor\log(T)\rfloor + 6$}
        \label{fig:comparisons-extra-rounds}
\end{figure}
\else
\begin{figure}
     \centering
     \begin{subfigure}[b]{0.45\textwidth}
         \centering
         \includegraphics[width=\textwidth]{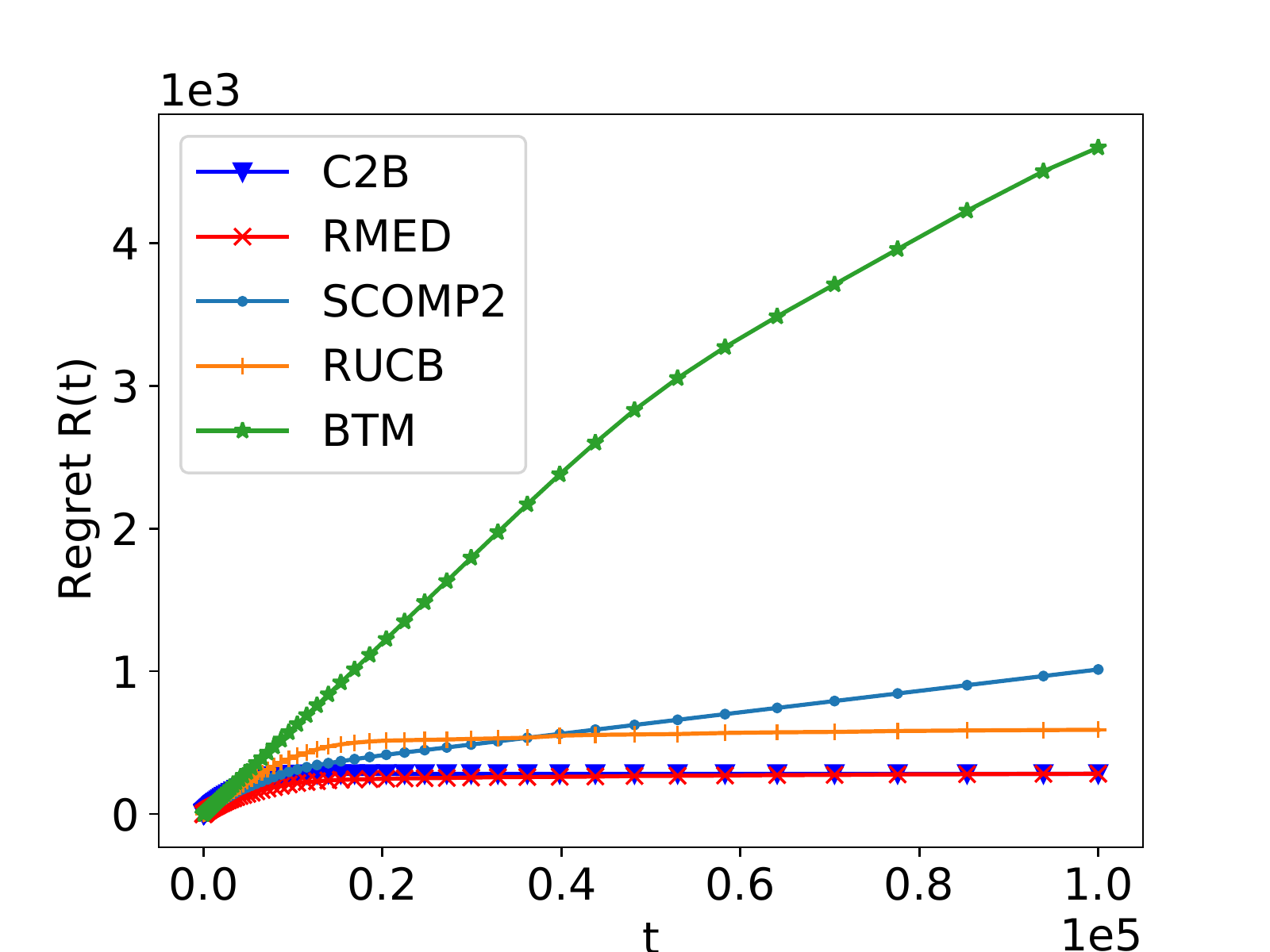}
         \caption{Six rankers}
     \end{subfigure}
     \begin{subfigure}[b]{0.45\textwidth}
         \centering
         \includegraphics[width=\textwidth]{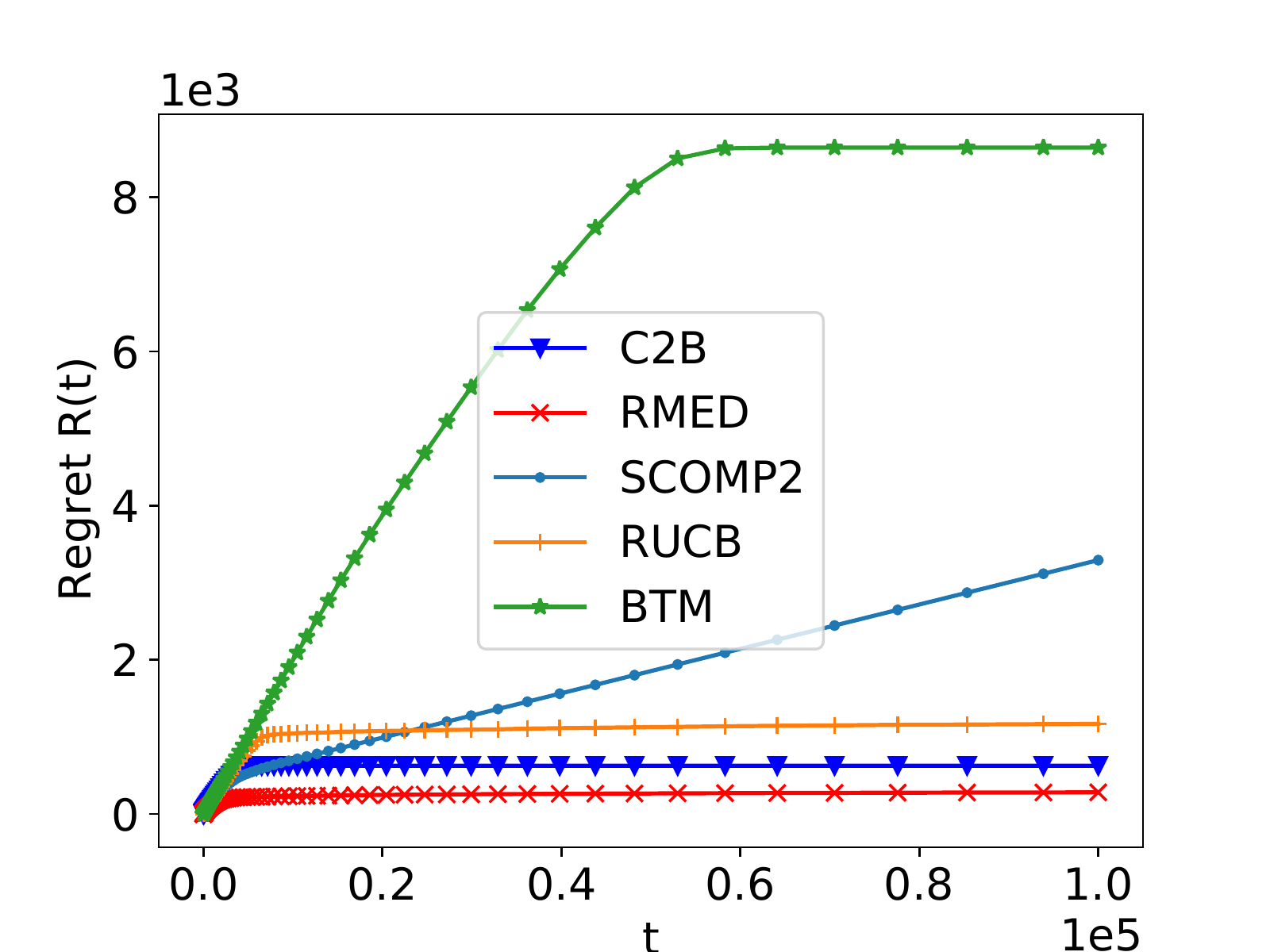}
         \caption{Sushi}
     \end{subfigure}
     \begin{subfigure}[b]{0.45\textwidth}
         \centering
         \includegraphics[width=\textwidth]{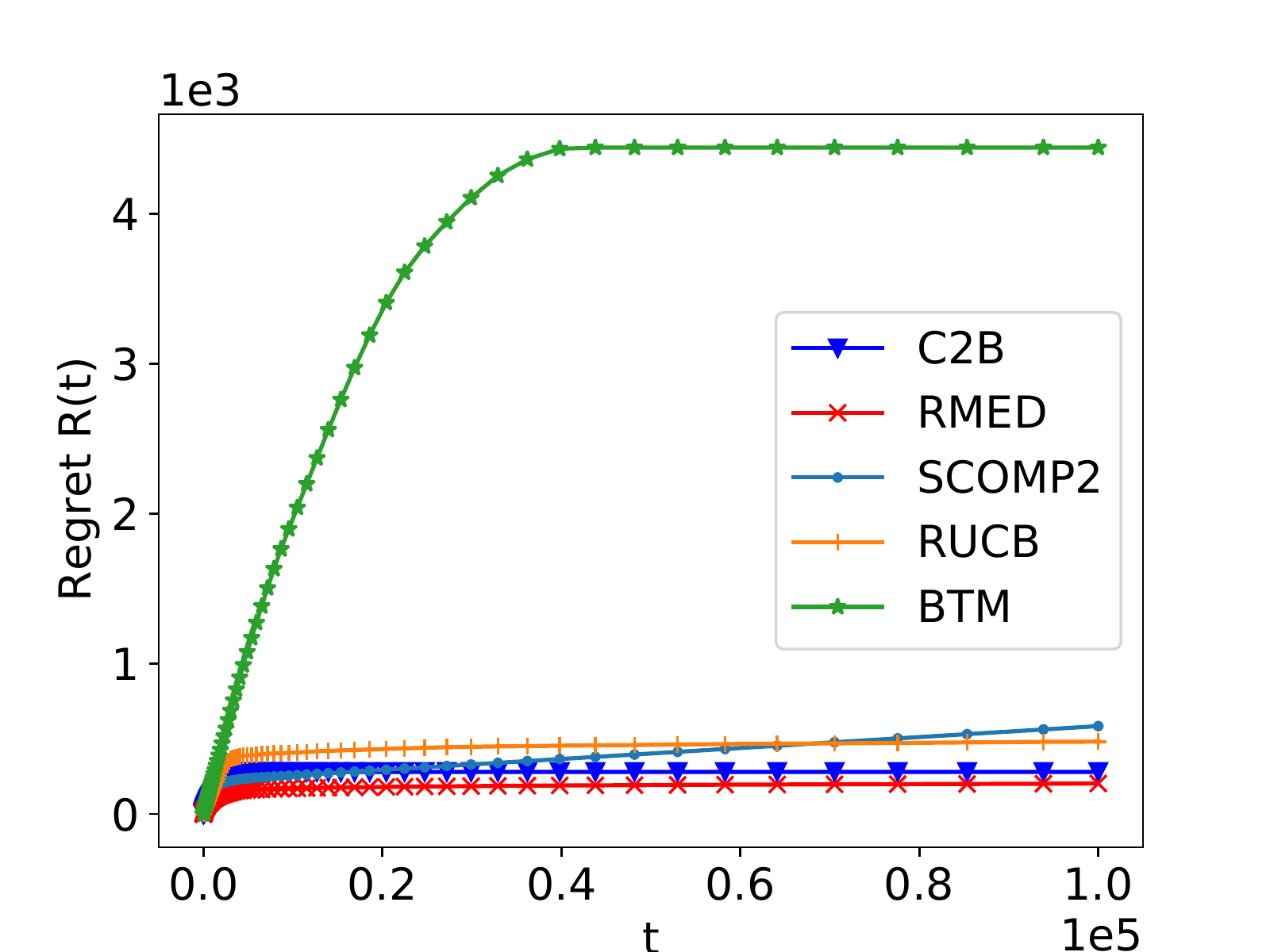}
         \caption{Irish-Meath}
     \end{subfigure}
     \begin{subfigure}[b]{0.45\textwidth}
         \centering
         \includegraphics[width=\textwidth]{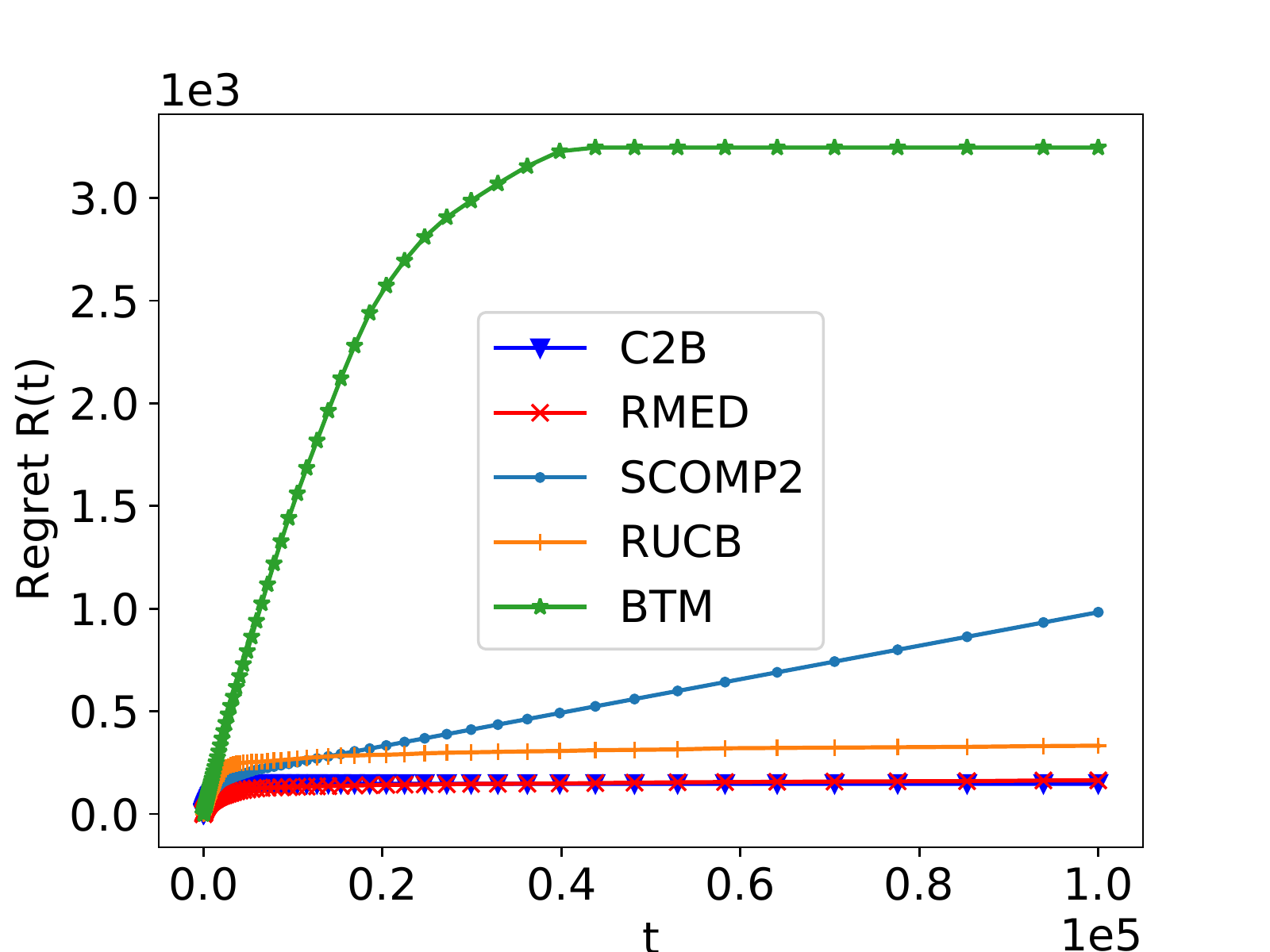}
         \caption{Irish -Dublin}
     \end{subfigure}
     \begin{subfigure}[b]{0.45\textwidth}
         \centering
         \includegraphics[width=\textwidth]{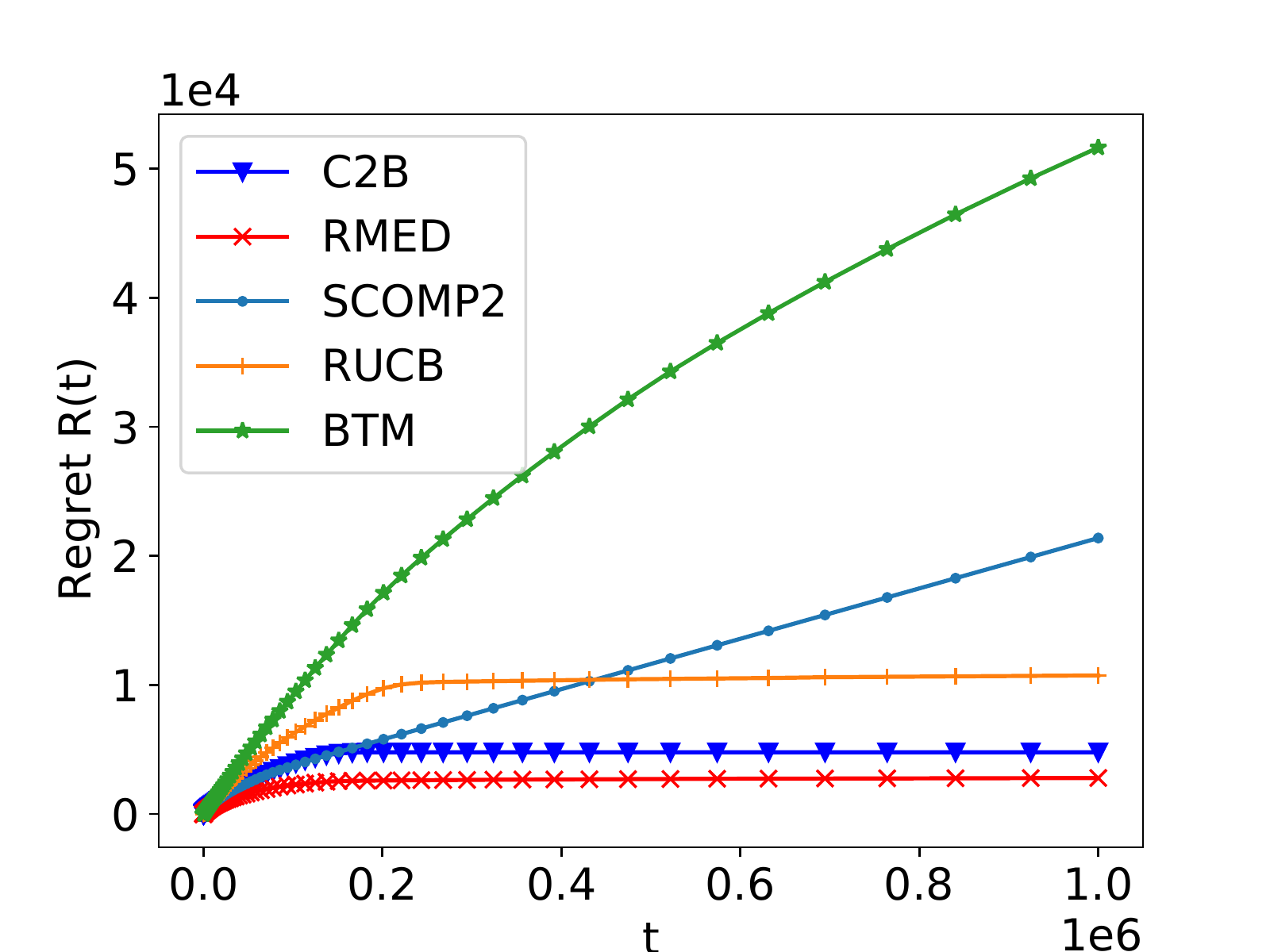}
         \caption{MSLR30}
     \end{subfigure}
     \begin{subfigure}[b]{0.45\textwidth}
         \centering
         \includegraphics[width=\textwidth]{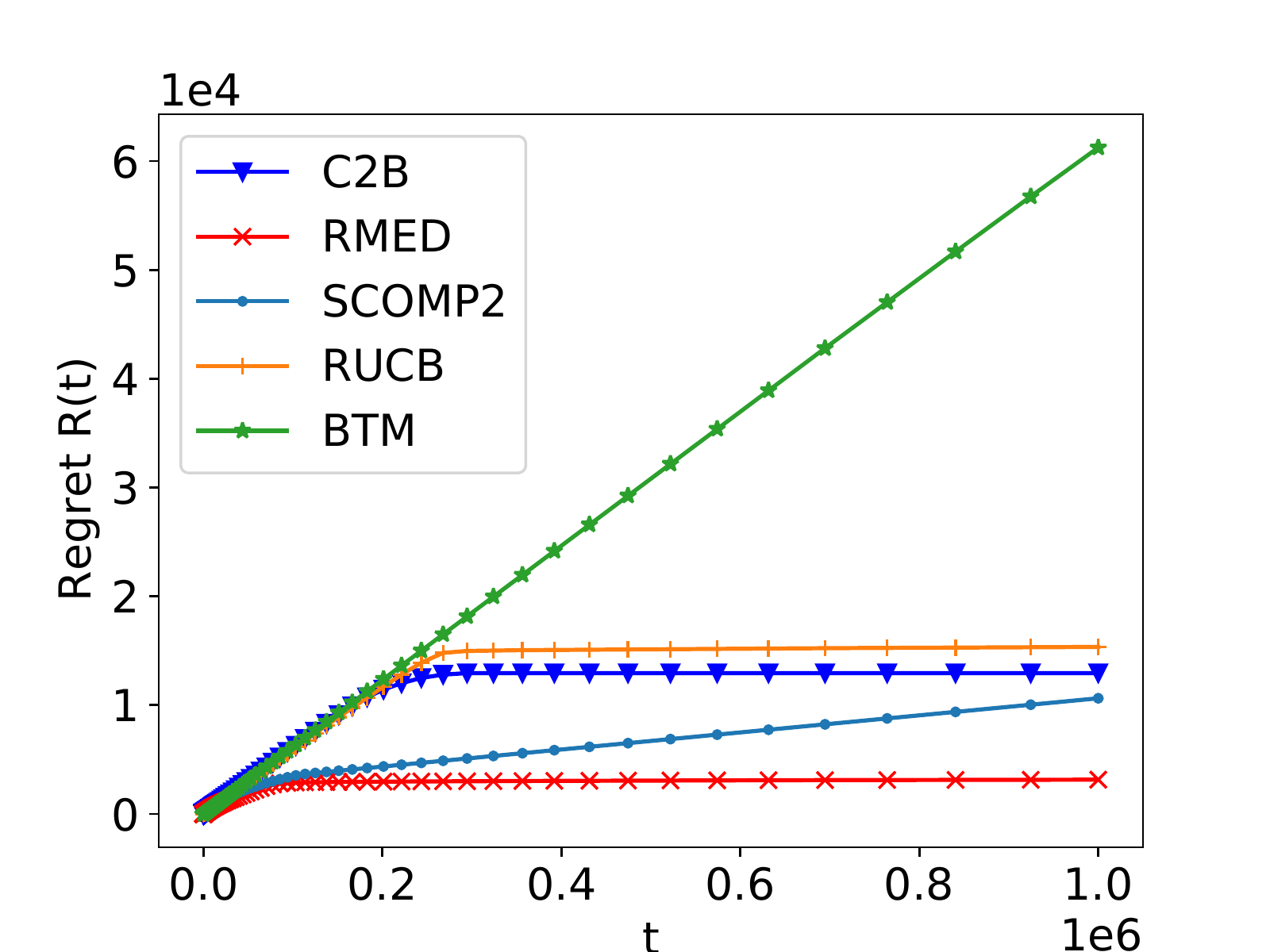}
         \caption{Yahoo30}
     \end{subfigure}
        \caption{Regret v/s t plots of algorithms when $B = \lfloor\log(T)\rfloor + 6$}
        \label{fig:comparisons-extra-rounds}
\end{figure}
\fi

\textbf{Computational Results.} As mentioned earlier, 
we compare our algorithms against a representative set of sequential dueling bandits algorithms
(RUCB, RMED, and BTM). 
We set $\alpha = 0.51$ for RUCB, and $f(K) = 0.3 K^{1.01}$ for RMED\ and \balg, and $\gamma = 1.3$ for BTM: these parameters are known to perform well both theoretically and empirically \cite{Komiyama+15a}.
We set $T = 10^6$ for {MSLR30} and {Yahoo30} datasets (as they have larger number of arms), and $T=10^5$ for the remaining four. For the first set of experiments, we set $B = \lfloor\log(T)\rfloor$.
We observe that \balg \ always outperforms BTM \ and beats SCOMP2 \ on most of the datasets. We observe that even when SCOMP2 \ beats \balg \ it has a slightly linear curve (implying that its regret would keep increasing as $T$ increases) while the regret curve of \balg \ is mostly flat. 
Furthermore, \balg \ performs comparably to RUCB \ in all datasets except Yahoo30.
We plot the results in \Cref{fig:comparisons}.
In the second set of experiments, we set $B = \lfloor \log(T) \rfloor + 6$. We observe that \balg \ always outperforms RUCB \ and, in fact, performs comparably to RMED \ on all datasets except Yahoo30.
We plot the results in \Cref{fig:comparisons-extra-rounds}.
Finally, we note that SCOMP2 exhibits varying performance across runs (even on the same dataset) and we  think that this is due to the randomness involved in selecting the ``seed set''.

\ifConfVersion
\vspace{-0.1in}
\else
\fi
\section{Conclusion}\label{sec:conclusion}
\ifConfVersion
\vspace{-0.1in}
\else
\fi
In this paper, we proposed a batched algorithm, named \balg, for the $K$-armed dueling bandit problem.
Assuming the existence of a Condorcet winner, we show both high-probability and expected regret bounds for \balg \  that trade-off smoothly with the number of batches.
Furthermore, we obtain asymptotic regret of $O(K^2\log^2(K)) + O(K\log(T))$ in $O(\log(T))$ batches,
nearly matching
the best regret bounds known in the fully sequential setting under the Condorcet condition.
Our computational results show that \balg, using $O(\log(T))$ batches, achieves almost the same performance as fully sequential algorithms over a variety of real-world datasets.
A direction for future research is to design batched algorithms for the $K$-armed dueling bandit problem when a Condorcet winner does not exist; for example, designing an algorithm for a more general concept of winner, such as \emph{Copeland winner} \cite{WuLiu16} or \emph{von Neumann winner} \cite{DudikHSSZ15}.

\appendix

\newcommand{\KL}{\text{KL}}
\newcommand{\balgkl}{$\mathtt{C2B}$-${\mathtt{KL}}$}

\section{The Batched Algorithm with KL-based Elimination Criterion}\label{sec:kl-alg}

In this section, we modify \balg \ to use a Kullback-Leibler divergence based elimination criterion. We provide a complete description of the algorithm, denoted \balgkl,
in Algorithm~\ref{alg:kl-b-round-condorcet}.
In what follows, we highlight the main differences of \balgkl \ from \balg.
Recall the following notation.
We use $\A$ to denote the current set of \emph{active} arms; i.e., the arms that have not been eliminated. We use index $r$ for rounds or batches. If pair $(i, j)$ is compared in round $r$, it is compared $q_r =  \lfloor q^r \rfloor $ times where $q = T^{1/B}$. We define the following quantities at the \emph{end} of each round $r$:
\begin{itemize}
\item $N_{i, j}(r)$ is the total number of times the pair $(i, j)$ has been compared.
\item $\widehat{p}_{i, j}({r})$ is the frequentist estimate of $p_{i, j}$, i.e.,  
\begin{align}
\label{eq:kl-pair_est}
\widehat{p}_{i, j}({r}) = \frac{\# \ i\text{ wins against } j \text{ until end of round } r}{N_{i, j}(r)}
    \,.
\end{align}

\item  A confidence-interval radius for each $(i,j)$ pair:
\begin{equation*}
c_{i, j}(r) = \sqrt{\frac{2\log(2K^2q_r)}{N_{i, j}(r)}}
\end{equation*} 

\item We define a term $I_j(r)$ which, at a high-level, measures how unlikely it is for $j$ to be the Condorcet winner at the end of batch $r$: 
$$ I_j(r) = \sum_{i: \widehat{p}_{i, j}(r) \geq \frac{1}{2}} D_{\KL}\left(\widehat{p}_{i, j}(r), \frac{1}{2}\right) \cdot N_{i, j}(r),$$ where $D_{\KL}(p, q)$ denotes the Kullback–Leibler divergence between two Bernoulli distributions: $B(p)$ and $B(q)$. We define $I^*(r) = \min_{j \in \A} I_j(r)$.

\end{itemize} 

The $B$-round algorithm, \balgkl, proceeds exactly as \balg. The only change is in the \emph{elimination criterion}, which we describe next. 

\paragraph{Elimination Criterion.} In round $r$, if, for any arm $j$, we have
$I_j(r) - I^*(r) > \log(T) + f(K)$, then  $j$ is eliminated from $\A$. Here $f(K)$ is a non-negative function of $K$, independent of $r$. \\

The main result of this section is to show that \balgkl \ achieves the following guarantee.

\begin{theorem}\label{thm:kl-condorcet-init-high-prob}
For any integer $B \geq 1$, there is an algorithm for the $K$-armed dueling bandit problem
that uses at most $B$ rounds. Furthermore,
for any $\delta > 0$, with probability at least $1 - \delta - \frac{1}{T}\cdot e^{K\log(C)-f(K)}$, where $C$ is some constant (see Lemma~\ref{lem:kl-not-deleted}), its regret under the Condorcet condition is at most 
\begin{align*}
R(T) &\leq O\left( T^{1/B} \cdot \frac{K^2 \log(K)}{\Delta_{\min}^2} \cdot \log\left(\frac{\log K}{\Delta_{\min}}\right) \right) &+\, O\left( T^{2/B} \cdot K^2\cdot \sqrt{\frac{1}{\delta}}\right) \,+\, \sum_{j \neq a^*} O\left(\frac{T^{1/B}\cdot \log( T)}{\Delta_j}\right) \\
& \ & + \sum_{j \neq a^*} O\left(\frac{T^{1/B} \cdot f(K)}{\Delta_j}\right)
\end{align*}
\end{theorem}

\paragraph{Remark.} Setting $f(K) > K\log(C)$, we get the same asymptotic expected regret bound as in Theorem~\ref{thm:condorcet-init-exp}. 
Following \cite{Komiyama+15a},  we set $f(K) = 0.3 K^{1.01}$ in our experiments.

We require the following result in the proof of Theorem~\ref{thm:kl-condorcet-init-high-prob}.
\begin{fact}\label{fact:diff-kl}
For any $\mu$ and $\mu_2$ satisfying $0 < \mu_2 < \mu < 1$. Let $C_1(\mu, \mu_2) = (\mu - \mu_2)^2/(2\mu(1-\mu_2))$. Then, for any $\mu_3 \leq \mu_2$, $$D_{\KL}(\mu_3, \mu) - D_{\KL}(\mu_3, \mu_2) \geq C_1(\mu, \mu_2) >0.$$

\end{fact}

\begin{algorithm}
\caption{\balgkl}
\label{alg:kl-b-round-condorcet}
\begin{algorithmic}[1]
\State \textbf{Input:} Arms $\B$, time-horizon $T$, integer $B \geq 1$
\State active arms $\A \gets \B$,  $r \gets 1$, emprical probabilities $\widehat{p}_{i, j}(0) = \frac{1}{2}$ for all $i, j\in \B^2$
\While{number of comparisons $\leq T$} 
\State \textbf{if} $\A = \{i\}$ for some $i$ \textbf{then} play $(i,i)$ for remaining trials
\State $D_r(i) \gets \{j \in \A : \widehat{p}_{i,j}(r-1) > \frac{1}{2} + c_{i, j}(r-1)\}$
\State $i_r \gets \arg\max_{i \in \A} |D_r(i)|$
\For{ $i \in \A \setminus \{i_r\}$}
\If{$i \in D_r(i_r)$}  
\State compare $(i_r, i)$ for $q_r$ times
\Else
\State for each $j \in \A$, compare $(i, j)$ for $q_r$ times  
\EndIf
\EndFor
\State compute $\widehat{p}_{i, j}(r)$ values
{
\If{$\exists j$ : $I_j(r) - I^*(r) > \log(T) + f(K)$} 
\State $\A \gets \A \setminus \{j\}$
\EndIf
}
\State $r \gets r+1$
\EndWhile
\end{algorithmic}
\end{algorithm}

\ignore{
\begin{theorem}\label{thm:high-prob}
Given any set $\B$ of $K$ arms and a time-horizon $T$, a positive integer $B \geq 1$, \balg \ uses at most $B$ rounds. Furthermore, for any $\delta > 0$, with probability at least $1 - \delta - \frac{1}{T}$, its regret   under the Condorcet 
assumption is at most 
$$ R(T) \leq  O\left( T^{1/B} \cdot \frac{K^2 \log(K)}{\Delta_{\min}^2} \cdot \log\left(\frac{\log K}{\Delta_{\min}}\right) \right) \,+\, O\left( T^{2/B} \cdot K^2\cdot \sqrt{\frac{1}{\delta}}\right) \,+\, \sum_{j \neq a^*} O\left(\frac{T^{1/B}\cdot \log(KT)}{\Delta_j}\right).$$
\end{theorem}

\begin{theorem}\label{thm:exp}
Given any set $\B$ of $K$ arms and a time-horizon $T$, a positive integer $B \geq 1$, \balg \ uses at most $B$ rounds and its expected regret  under the Condorcet 
condition is at most 
$$\E[R(T)] =  O\left( T^{1/B} \cdot \frac{K^2 \log(K)}{\Delta_{\min}^2} \cdot \log\left(\frac{\log K}{\Delta_{\min}}\right) \right) \,+\, O\left( T^{2/B} \cdot K^2\right) \,+\, \sum_{j \neq a^*} O\left(\frac{T^{1/B}\cdot \log(KT)}{\Delta_j}\right).$$
\end{theorem}
}

The high-level outline of the analysis is exactly the same as that of \balg. For completeness, we provide the analysis in the following section; however, we skip the proofs of lemmas that follow from the analysis of \balg.

\subsection{The Analysis}
In this section, we prove the high-probability regret bound for \balgkl.
Recall that $q = T^{1/B}$, and that $q \geq 2$. 
We first show that, with high probability, $a^*$ is not eliminated during the execution of the algorithm. The following lemma formalizes this.

\begin{lemma}\label{lem:kl-not-deleted}
Let $G$ denote the event that the best arm $a^*$ is not eliminated during the execution of \balgkl. We can bound the probability of $\overline{G}$ as follows. $$\pr(\overline{G}) \leq \frac{1}{T}\cdot e^{K\log(C)-f(K)},$$ where $C = \max_j C(j) + 1$, is a constant, with $C(j) = \left(\frac{1}{e^{ D_{\KL}\left(p_{j, a^*}, 1/2\right)}-1} + \frac{e^{C_1\left(p_{a^*,j}, 1/2\right)} }{\left(e^{C_1\left(p_{a^*,j}, 1/2\right)} -1\right)^2}\right)$.
\end{lemma}
\ifConfVersion
\vspace{-0.2in}
\else\fi
\begin{proof}

Let $n_j$ denote the number of times $a^*$ and $j$ are compared. Let $\widehat{p}_{a^*, j}(n_j)$ denote the frequentist estimate of ${p}_{a^*, j}$ when $a^*$ and $j$ are compared $n_j$ times (we will abuse notation and use $\widehat{p}_{a^*, j}$ when $n_j$ is clear from context). Let $S \in 2^{[K]\setminus \{a^*\}} \setminus \emptyset$, and consider vector $\{n_j \in \mathbb{N} : j \in S\}$. 
We define $A = \sum_{j \in S} D_{\KL}\left(\widehat{p}_{j,a^*}, 1/2\right) \cdot n_j$.
Let $D(S; \{n_j : j \in S\})$ denote the event that $a^*$ and $j$ are compared $n_j$ times and $\widehat{p}_{a^*, j} \leq 1/2$ for all $j \in S$, and that $A > \log(T) + f(K)$. The probability of this event upper bounds the probability that $a^*$ is eliminated (as per our elimination criterion) when $a^*$ and $j$ are compared $n_j$ times, and $\widehat{p}_{a^*, j} \leq 1/2$ for all $j \in S$. We will show that 
\begin{equation}\label{eq:kl-prob}
    \pr(D(S; \{n_j : j \in S\})) \leq \frac{e^{-f(K)}}{T}\prod_{j \in S} \left( e^{-n_j D_{\KL}\left(p_{j, a^*}, 1/2\right)} + n_j e^{C_1\left(p_{j, a^*}, 1/2\right)}  \right)
\end{equation}
 where $C_1(\mu_1, \mu_2) = (\mu_1-\mu_2)^2/(2\mu_1(1-\mu_2))$.
Using the above, we first show that by taking a union bound over all $S \in 2^{[K]\setminus \{a^*\}} \setminus \emptyset$ and $\{n_j : j \in S\}$, we obtain the final result. We have
\begin{align}
    \pr(\overline{G}) &\leq  \sum_{S \in 2^{[K]\setminus \{a^*\}} \setminus \emptyset} \sum_{n_j \in \mathbb{N}^{|S|}} \pr(D(S; \{n_j : j \in S\})) \notag \\
    &\leq \sum_{S \in 2^{[K]\setminus \{a^*\}} \setminus \emptyset} \sum_{n_j \in \mathbb{N}^{|S|}} \frac{e^{-f(K)}}{T}\prod_{j \in S} \left( e^{-n_j D_{\KL}\left(p_{j, a^*}, 1/2\right)} + n_j e^{C_1\left(p_{j, a^*}, 1/2\right)}  \right)\notag\\
    &=\frac{e^{-f(K)}}{T}\sum_{S \in 2^{[K]\setminus \{a^*\}} \setminus \emptyset} \prod_{j \in S} \sum_{n_j \in \mathbb{N}}  \left( e^{-n_j D_{\KL}\left(p_{j, a^*}, 1/2\right)} + n_j e^{C_1\left(p_{j, a^*}, 1/2\right)}  \right) \label{eq:kl-a1-1}\\
    &=\frac{e^{-f(K)}}{T}\sum_{S \in 2^{[K]\setminus \{a^*\}} \setminus \emptyset} \prod_{j \in S} \left(\frac{1}{e^{D_{\KL}\left(p_{j, a^*}, 1/2\right)}-1} + \frac{e^{C_1\left(p_{j, a^*}, 1/2\right)} }{\left(e^{C_1\left(p_{j, a^*}, 1/2\right)} -1\right)^2}\right)\label{eq:kl-a1-2}\\
    &\leq \frac{e^{-f(K)}}{T}\sum_{S \in 2^{[K]\setminus \{a^*\}} \setminus \emptyset} (C-1)^{|S|} \leq \frac{e^{-f(K)}}{T} \cdot C^K \label{eq:kl-a1-3} \\
    &= \frac{1}{T}\cdot e^{K\log(C)-f(K)}\notag
\end{align}
where \eqref{eq:kl-a1-1} follows by swapping the order of summation and multiplication, 
\eqref{eq:kl-a1-2} uses $\sum_{n=1}^{\infty}e^{-nx} = 1/(e^x-1)$ and $\sum_{n=1}^{\infty}ne^{-nx} = e^x/(e^x-1)^2$, 
and \eqref{eq:kl-a1-3} follows by letting \\ $C(j) = \left(\frac{1}{e^{D_{\KL}\left(p_{j, a^*}, 1/2\right)}-1} + \frac{e^{C_1\left(p_{j, a^*}, 1/2\right)} }{\left(e^{C_1\left(p_{j, a^*}, 1/2\right)} -1\right)^2}\right)$, $C = \max_j C(j) + 1$ and the binomial theorem. 
To complete the proof, we need to prove \eqref{eq:kl-prob}.

For the remainder of this proof, we fix $S \in 2^{[K]\setminus \{a^*\}} \setminus \emptyset$, and vector $\{n_j \in \mathbb{N} : j \in S\}$. 
Observe that $$ \pr(D(S; \{n_j : j \in S\})) = \pr\left(A > \log(T) + f(K)\right) = \pr\left(T < e^{-f(K)}\cdot e^{A}\right) $$
where we defined $A = \sum_{j \in S} D_{\KL}\left(\widehat{p}_{j,a^*}, 1/2\right) \cdot n_j$. 
By Markov's inequality, we have 
\begin{equation}\label{eq:kl-not-deleted-1}
\pr\left(e^{-f(K)}\cdot e^{A} > T\right) \leq \frac{\E[e^{-f(K)}\cdot e^{A}]}{T} = \frac{e^{-f(K)}}{T}\cdot \E[e^{A}] 
\end{equation}
where the last equality follows since $f(K)$ is constant (with respect to $\{n_j\}$ values). 
So, it suffices to bound $\E[e^{A}]$. 
Towards this end, we define the following term: $$P_j(x_j) = \pr\left(\widehat{p}_{j,a^*} \geq \frac{1}{2} \text{ and } D_{\KL}\left(\widehat{p}_{j,a^*}, \frac{1}{2}\right)  \geq x_j\right).$$

Then, we have
\begin{align}
     \E[e^{A}] &=  \int_{\{x_j\} \in [0, \log(2)]^{|S|}} \exp\left(\sum_{j \in S}n_jx_j\right) \prod_{j \in S} d(-P_j(x_j)) \notag \\
     &= \prod_{j\in S} \int_{x_j \in [0, \log 2]} e^{n_jx_j} d(-P_j(x_j)) \label{eq:kl-a1-4} \\
     &= \prod_{j\in S}\left([ -e^{n_jx_j}P_j(x_j)]_0^{\log(2)} +  \int_{x_j \in [0, \log(2)]} n_je^{n_jx_j}P_j(x_j) dx_j \right) \label{eq:kl-a1-5} \\
     &=\prod_{j \in S} \left( P_j(0)+ \int_{x_j \in [0, \log(2)]}n_j e^{n_jx_j}P_j(x_j)dx_j\right) \notag \\
     &\leq \prod_{j \in S} \left( e^{-n_j D_{\KL}\left(p_{j, a^*}, 1/2\right)}  + \int_{x_j \in [0, \log(2)]}n_j e^{n_jx_j}e^{-n_j\left(x_j + C_1\left(p_{j, a^*},1/2\right)\right)} dx_j \right) \label{eq:kl-a1-6} \\
     &= \prod_{j \in S} \left( e^{-n_j D_{\KL}\left(p_{j, a^*}, 1/2\right)}  + \int_{x_j \in [0, \log(2)]}n_j e^{ C_1\left(p_{j, a^*},1/2\right)} dx_j \right) \notag \\
     &\leq \prod_{j \in S} \left( e^{-n_j D_{\KL}\left(p_{j, a^*}, 1/2\right)}  + n_j e^{ C_1\left(p_{j, a^*},1/2\right)}  \right)\notag
\end{align}
where \eqref{eq:kl-a1-4} follows from the independence of the comparisons. We obtain \eqref{eq:kl-a1-5} by applying integration by parts, \eqref{eq:kl-a1-6} follows from the Chernoff bound and 
Fact~\ref{fact:diff-kl}; here  $C_1(\mu_1, \mu_2) = (\mu_1-\mu_2)^2/(2\mu_1(1-\mu_2))$, and the final inequality follows by observing that $ \int_{x_j \in [0, \log(2)]}n_j e^{ C_1\left(p_{j, a^*},1/2\right)} dx_j = n_j e^{ C_1\left(p_{j, a^*},1/2\right)}\cdot \int_{x_j \in [0, \log(2)]} dx_j = n_j e^{ C_1\left(p_{j, a^*},1/2\right)} \log(2)$. Note that $\log$ refers to the natural logarithm, so we have 
$\log(2) \leq 1$. Combined with \eqref{eq:kl-not-deleted-1}, this completes the proof of \eqref{eq:kl-prob}.
\end{proof}

\subsubsection{High-probability Regret Bound}
We now prove Theorem~\ref{thm:kl-condorcet-init-high-prob}. Fix any $\delta>0$. We first define event $E(\delta)$ as before.
\begin{definition}[Event $E(\delta)$]
An estimate $\widehat{p}_{i,j}(r)$ in batch $r$ is {\bf \emph{weakly-correct}} if $|\widehat{p}_{i, j}(r) - p_{i, j}| \leq c_{i, j}(r)$. Let $ C(\delta) := \lceil \frac{1}{2} \log_q(1/\delta)\rceil$.  
We say that event $E(\delta)$ occurs if for each batch $r\ge C(\delta)$, every estimate  is weakly-correct.
\end{definition}

The next lemma shows that $E(\delta)$ occurs with probability at least $1-\delta$. Since $E(\delta)$ does not depend on the elimination criterion, its proof follows from the analysis of \balg.
\begin{lemma}\label{lem:kl-c-delta-conf}
For all $\delta > 0$, we have 
 $$ \pr(\neg E(\delta)) \,\,= \,\, \pr\left( \exists r \ge C(\delta), i, j : |\widehat{p}_{i, j}(r) - p_{i, j}| > c_{i, j}(r) \right)  \,\,\leq  \,\,\delta.$$ 
\end{lemma}

As before, we analyze our algorithm under both events $G$ and $E(\delta)$. Recall that, under event $G$, the best arm $a^*$ is not eliminated. 
\emph{Conditioned on these}, we next show:
\begin{itemize}
    \item The best arm, $a^*$, is \emph{not defeated} by any arm $i$ in any round $r > C(\delta)$ (\Cref{lem:kl-not-defeated}).
    \item Furthermore, there exists a round $\rdel\ge C(\delta)$ such that arm $a^*$ defeats \emph{every other arm}, in every round after $\rdel$  (\Cref{lem:kl-trapped}).
\end{itemize}

We re-state the formal lemmas next.

\begin{lemma}\label{lem:kl-not-defeated}
Conditioned on $G$ and $E(\delta)$, for any round $r > C(\delta)$, arm  $a^*$ is not defeated by any other arm, i.e.,  $a^* \notin \cup_{i\ne a^*} D_r(i)$.
\end{lemma}

To proceed, we need the following definitions.
\begin{definition}
The candidate $i_r$  of round $r$ is called the {\bf \emph{champion}} 
if $|D_r(i_r)| = |\A| - 1$; that is, if  $i_r$ defeats every other active arm.
\end{definition}

\begin{definition}\label{def:kl-r-star}
Let $\rdel\ge C(\delta)+1$ be the smallest integer such that 
\begin{equation*}
q^{\rdel} \ge  2A\log A,\qquad \mbox{where }A:=\frac{32}{\Delta_{\min}^2}\cdot \log(2 K^2).
\end{equation*}
\end{definition}
We use the following inequality based on this choice of $\rdel$.
\begin{lemma}\label{lem:kl-rdel}
The above choice of $\rdel$ satisfies
$$ q^{r} > \frac{8}{\Delta_{\min}^2}\cdot \log\left(2 K^2 q_{r}\right),\qquad \forall r\ge \rdel . $$
\end{lemma}

Then, we have the following.

\begin{lemma}\label{lem:kl-trapped}
Conditioned on $G$ and $E(\delta)$, the best arm $a^*$ is the champion in every round $r>\rdel$. 
\end{lemma}

We are now ready to prove \Cref{thm:kl-condorcet-init-high-prob}.

\begin{proof}[Proof of \Cref{thm:kl-condorcet-init-high-prob}]
First, recall that in round $r$ of \balg, any pair is compared $q_r =  \lfloor q^r \rfloor$ times where $q = T^{1/B}$. Since $q^B = T$, \balg \ uses at most $B$ rounds.

For the rest of proof, we fix $\delta > 0$. We now analyze the regret incurred by \balg, conditioned on events $G$ and $E(\delta)$. Recall that $\pr(G) \geq 1 - \frac{1}{T}\cdot e^{K\log(C)-f(K)}$ (\Cref{lem:kl-not-deleted}), and $\pr(E(\delta)) \geq 1 - \delta$ (\Cref{lem:kl-c-delta-conf}). Thus, $\pr(G\cap E(\delta)) \geq 1 - \delta - \frac{1}{T}\cdot e^{K\log(C)-f(K)}$. Let $R_1$ and $R_2$ denote the regret incurred before and after round $\rdel$ (see \Cref{def:kl-r-star}) respectively. 

\paragraph{Bounding $R_1$.} We can bound $R_1$ as in the proof of Theorem~\ref{thm:condorcet-init-high-prob}; so, we get

\begin{equation}\label{eq:kl-r1}  
 R_1 \le O(K^2)\cdot \max\left\{ q\cdot \frac{\log K}{\Delta_{\min}^2}\cdot \log\left(\frac{\log K}{\Delta_{\min}}\right) \, ,\, q^2\sqrt{\frac{1}{\delta}}\right\}.
\end{equation}

\paragraph{Bounding $R_2$.} This is the regret in rounds $r\ge \rdel+1$. By Lemma~\ref{lem:kl-trapped}, arm $a^*$ is the champion in all these rounds. So, the only comparisons in these rounds are of the form $(a^*,j)$ for $j\in \A$. 

Consider any arm $j\ne a^*$. Let $T_j$ be  the total number of comparisons that $j$ participates in after round $\rdel$. Let $r$ be the penultimate round that $j$ is played in. We can assume that $r\ge \rdel$ (otherwise arm $j$ will never participate in rounds after $\rdel$, i.e., $T_j=0$). 
As arm $j$ is {\em not} eliminated after round $r$,  
$$I_j(r) - I^*(r) \leq \log(T) + f(K).$$
By \Cref{lem:kl-trapped}, $I^*(r) = 0$ (since $a^*$ is the \emph{champion}, the summation is empty). So, we have $I_j(r) \leq \log(T) + f(K).$ 
Observe that 
\begin{equation}\label{eq:kl-lb}
I_j(r) \geq D_{\KL}\left(\widehat{p}_{a^*, j}(r), \frac{1}{2}\right) N_{a^*, j}(r)
\end{equation}
We can lower bound $D_{\KL}\left(\widehat{p}_{a^*, j}(r), \frac{1}{2}\right)$ as follows.
\begin{equation*}
   D_{\KL}\left(\widehat{p}_{a^*, j}(r), \frac{1}{2}\right) \geq \left(\widehat{p}_{a^*, j}(r) - \frac{1}{2}\right)^2 
        \geq \left(p_{a^*, j} - c_{a^*, j}(r) - \frac{1}{2}\right)^2 
        \geq \left(\frac{\Delta_j}{2}\right)^2 
\end{equation*}
where the first inequality follows from Pinsker's inequality, the second inequality uses \Cref{lem:kl-c-delta-conf} and the final inequality uses the fact that $c_{a^*, j}(r) \leq \frac{\Delta_{\min}}{2}$, which follows by the choice of $r(\delta)$. Plugging this into \eqref{eq:kl-lb}, we get $$ \frac{\Delta_j^2}{4} \cdot N_{a^*, j}(r) \leq \log(T) + f(K) $$ which on re-arranging gives
$$ N_{a^*, j}(r) \leq \frac{4(\log(T) + f(K))}{\Delta_j^2}.$$ As $r+1$ is the last round that $j$ is played in, and $j$ is only compared to $a^*$ in each round after $\rdel$, 
$$ T_j \ \leq \ N_{a^*, j}(r+1) \ \leq \  N_{a^*, j}(r) + 2q\cdot N_{a^*, j}(r) \ \leq \  \frac{12q\cdot(\log(T) + f(K))}{\Delta_j^2}.$$
The second inequality follows since $j$ is compared to $a^*$ in rounds $r$ and $r+1$, and the number of comparisons in round $r+1$ is $\lfloor q^{r+1} \rfloor \leq q \cdot (2 q_r) \leq 2q \cdot N_{a^*, j}(r)$.
Adding over all arms $j$, the total regret accumulated beyond round $\rdel$ is 
\begin{equation}\label{eq:kl-r2}
R_2 = \sum_{j \neq a^*} T_j \Delta_j \leq  \sum_{j \neq a^*} O\left(\frac{q\cdot (\log(T) + f(K))}{\Delta_j}\right). 
\end{equation}

\noindent Combining \eqref{eq:kl-r1} and \eqref{eq:kl-r2}, and using $q=T^{1/B}$, we obtain 
\begin{align*}
R(T) &\leq O\left( T^{1/B} \cdot \frac{K^2 \log(K)}{\Delta_{\min}^2} \cdot \log\left(\frac{\log K}{\Delta_{\min}}\right) \right) &+\, O\left( T^{2/B} \cdot K^2\cdot \sqrt{\frac{1}{\delta}}\right) \,+\, \sum_{j \neq a^*} O\left(\frac{T^{1/B}\cdot \log( T)}{\Delta_j}\right) \\
& \ & + \sum_{j \neq a^*} O\left(\frac{T^{1/B} \cdot f(K)}{\Delta_j}\right)
\end{align*}
 This completes the proof \Cref{thm:kl-condorcet-init-high-prob}.
\end{proof}

{
\bibliographystyle{abbrv}
\bibliography{ref}
}

\end{document}